\documentclass[11pt]{article}
\usepackage[letterpaper]{geometry}

\usepackage{times}
\usepackage{graphicx} 
\usepackage{subfigure}

\usepackage{algorithm}
\usepackage{algorithmic}
\usepackage{amsmath, comment,amssymb,xspace,multirow,bbm}
\usepackage{float}

\usepackage{amsthm}
\newtheorem{lemma}{Lemma}
\newtheorem{theorem}{Theorem}
\newtheorem{definition}{Definition}
\newtheorem{proposition}{Proposition}

\usepackage{authblk}

\DeclareMathOperator*{\argmin}{\arg\min}
\DeclareMathOperator*{\eql}{:=}

\newcounter{tightlistcnt}
\newenvironment{tightitemize}{%
    \begin{list}{\textbullet}{\usecounter{tightlistcnt}%
    \topsep 0in
    \partopsep 0in
    \itemsep .2em
    \parsep 0in
    \leftmargin 0em
    \rightmargin 0in
    \listparindent 0em
    \itemindent .75em
    \labelsep .75em
    \labelwidth 0in
    }%
}{%
    \end{list}%
}

\newenvironment{tightenumerate}{%
    \begin{list}{\arabic{tightlistcnt}.}{\usecounter{tightlistcnt}
    \topsep 0in
    \partopsep 0in
    \itemsep .2em
    \parsep 0in
    \leftmargin 0.0em
    \rightmargin 0in
    \listparindent 0em
    \itemindent .75em
    \labelsep .75em
    \labelwidth 0in
    }%
}{%
    \end{list}%
}

\begin{document}

\title{Collaborative Filtering with Stability}

\author[1]{Dongsheng Li}
\author[1]{Chao Chen}
\author[2]{Qin Lv}
\author[1]{Junchi Yan}
\author[2]{Li Shang}
\author[1]{Stephen M. Chu}
\affil[1]{IBM Research - China, Shanghai P. R. China 201203}
\affil[2]{University of Colorado Boulder, Boulder, Colorado USA 80309}

\maketitle

\begin{abstract}
Collaborative filtering (CF) is a popular technique in today's recommender systems,
and matrix approximation-based CF methods have achieved great success in both
rating prediction and top-N recommendation tasks.
However, real-world user-item rating matrices are typically sparse, incomplete and noisy,
which introduce challenges to the algorithm stability of matrix approximation, i.e., small changes
in the training data may significantly change the models. As a result, existing 
matrix approximation solutions yield low generalization performance, exhibiting
high error variance on the training data, and minimizing the training error may
not guarantee error reduction on the test data.
This paper investigates the algorithm stability problem of matrix
approximation methods and how to achieve stable collaborative filtering
via stable matrix approximation. We present a new algorithm design framework, which
(1) introduces new optimization objectives to guide stable matrix
approximation algorithm design, and (2) solves the optimization problem to obtain
stable approximation solutions with good generalization performance.
Experimental results on real-world datasets demonstrate
that the proposed method can achieve better accuracy compared with
state-of-the-art matrix approximation methods and ensemble methods
in both rating prediction and top-N recommendation tasks.
\end{abstract}


\section{Introduction}
\label{sec:intro}
Recommender systems have become essential components for many online
applications~\cite{linden2003amazon,Das2007}, and collaborative filtering (CF) methods
are popular in today's recommender systems due to its superior accuracy~\cite{adomavicius2005toward}.
Matrix approximation (MA) has been widely adopted in recommendation tasks on
both rating prediction~\cite{paterek2007improving,lee2013local,Beutel15,chen2016mpma}
and top-N recommendation~\cite{Hu08,Pan08,rendle09},
and achieved state-of-the-art recommendation accuracy~\cite{koren2009matrix}.
In MA-based CF methods~\cite{paterek2007improving,koren2009matrix},
a given user-item rating matrix is approximated using observed ratings (often sparse)
to obtain low dimensional user features and item features,
then user ratings on unrated items are predicted using the dot product of
corresponding user and item feature vectors. 
MA methods have the capability of reducing the dimensionality of
user/item rating vectors, hence are suitable for collaborative filtering applications with sparse
data~\cite{koren2009matrix}. Indeed, MA-based collaborative filtering has been widely
used in existing recommendation solutions~\cite{paterek2007improving,koren2009matrix,lee2013local,Beutel15,Hu08,Pan08,rendle09}.

The sparsity of the data, incomplete and noisy~\cite{Keshavan10,Candes12},
introduces challenges to the algorithm stability of matrix approximation-based
collaborative filtering methods. In MA-based CF methods,
models are easily biased due to the limited training data (sparse), and small
changes in the training data (noisy) can significantly change the models.
As demonstrated in this work, existing MA-based methods cannot provide
stable matrix approximations and hence stable collaborative filtering.
Such unstable matrix approximations introduce
high training error variance, and minimizing the training error may not
guarantee consistent error reduction on the test data, i.e.,
low generalization performance~\cite{Bousquet01,Srebro04,Srebro04MMMF}.
As such, the algorithm stability has direct
impact on generalization performance~\cite{Hardt15}, and an unstable MA method has low
generalization performance~\cite{Srebro04,Hardt15}.

Heuristic techniques, such as cross-validation and ensemble
learning~\cite{Koren08,mackey2011divide,lee2013local,Chen15},
can be adopted to improve the generalization performance of MA-based CF methods.
However, cross validation methods have the drawback that the amount of data
available for model learning is reduced~\cite{Kohavi95,Bousquet01}.
Ensemble MA methods~\cite{lee2013local,Chen15,li2017mrma} are computationally expensive
due to the training of sub-models.
Recently, the notion of ``algorithmic stability'' has been introduced to
investigate the theoretical bound of the generalization performance of learning
algorithms~\cite{Bousquet01,Bousquet02,Agarwal09,Shalev-Shwartz10,London13,Hardt15}.
It is timely to develop stable matrix approximation methods with low generation errors,
which are suitable for collaborative filtering applications with sparse, incomplete and noisy data.

This paper extends a stable matrix approximation algorithm design framework~\cite{li2016low}
to achieve stable collaborative filtering on both rating prediction and top-N recommendation.
It formulates new optimization objectives to derive stable matrix approximation algorithms,
namely SMA,
and solves the new optimization objectives to obtain SMA solutions with good
generalization performance. We first introduce the stability notion in MA,
and then develop theoretical guidelines for deriving MA solutions with high
stability.  Then, we formulate a new optimization problem for achieving stable
matrix approximation, in which minimizing the loss function can obtain solutions with high
stability, i.e., good generalization performance. Finally,
we develop a stochastic gradient descent method to solve the new optimization
problem. Experimental results on real-world datasets
demonstrate that the proposed SMA method can deliver a stable MA algorithm,
which achieves better accuracy
over state-of-the-art single MA methods and ensemble MA methods in both rating prediction
and top-N recommendation tasks.
The key contributions of this paper are as follows:
\begin{enumerate}
\item This work introduces the stability concept in matrix approximation methods,
which can provide theoretical guidelines for deriving stable matrix approximation methods to achieve stable collaborative filtering;
\item A stable matrix approximation algorithm design framework is proposed, which can
achieve high stability, i.e., high generalization performance by
designing and solving new optimization objectives derived based on stability analysis;
\item Evaluation using real-world datasets demonstrates that the proposed method
can make significant improvement in recommendation accuracy over state-of-the-art
matrix approximation-based collaborative filtering methods in both rating prediction
and top-N recommendation tasks.
\end{enumerate}

The rest of this paper is organized as follows:
Section~\ref{sec:problem} formulates the stability problem in matrix approximation.
Section~\ref{sec:theory} analyzes the stability of matrix approximation and formally proves our key observations.
Section~\ref{sec:algorithm} presents details of the proposed stable matrix approximation method
on both rating prediction task and top-N recommendation task.
Section~\ref{sec:experiment} presents the experimental results.
Section~\ref{sec:related} discusses related work, and
we conclude this work in Section~\ref{sec:conclusion}.

\section{Problem Formulation}
\label{sec:problem}
This section first summarizes matrix approximation methods, 
and then introduces the definition of stability w.r.t. matrix approximation. 
Next, we conduct quantitative analysis of the relationship between
algorithm stability and generalization error, and empirically
demonstrate that models with high stability will generalize well.

\subsection{Matrix Approximation}
In this paper, upper case letters, such as $R, U, V$ denote matrices.
For a targeted matrix $R\in\mathbb{R}^{m\times n}$, $\Omega$ denotes
the set of observed entries in $R$, and $\hat{R}$ denotes the approximation of $R$.
Generally, loss functions should be defined towards different matrix approximation tasks,
and $\hat{R}$ is determined by minimizing such loss functions~\cite{lee2001algorithms,mnih2007probabilistic,salakhutdinov2008bayesian,Yan10}.
Let loss function $Loss(R,\hat{R})$ be the error of approximating $R$ by $\hat{R}$,
then the general objective of matrix approximation can be described as follows:
\begin{equation}
\label{eqn:loss}
\hat{R} = \arg\min_{X}{Loss(R,X)}.
\end{equation}
$Loss(R,X)$ should vary for different tasks. For instance,
incomplete Singular Value Decomposition (SVD) usually adopts Frobenius norm~\cite{lee2013local}
to define loss function, and Compressed Sensing adopts nuclear norm~\cite{Donoho06}.

Among existing matrix approximation methods, the rank of $\hat{R}$ --- $r$ is
considered low in many scenarios, because $r\ll\min\{m,n\}$ can achieve good performance
in many collaborative filtering applications. This kind of matrix approximation methods is called
low-rank matrix approximation (LRMA).
The objective of $r$-rank matrix approximation is to determine two feature matrices, i.e.,
$U\in\mathbb{R}^{m\times r}, V\in\mathbb{R}^{n\times r}$, $r\ll\min\{m,n\}$, s.t.,
$\hat{R} = UV^T$.
Generally, the optimization problem of LRMA can be formally described as follows:
\begin{equation}
\label{eqn:loss_lr}
\hat{R} = \arg\min_{X}{Loss(R,X)},~s.t.,~rank(X)=r.
\end{equation}
Typically, the problems defined by Equation~\ref{eqn:loss} and Equation~\ref{eqn:loss_lr}
are often difficult non-convex optimization problems. Therefore, iterative methods such as
stochastic gradient descent (SGD) are usually adopted to
find solutions that will converge to local minimum~\cite{lee2001algorithms,mnih2007probabilistic}.


\subsection{Stability w.r.t Matrix Approximation}
Recent work on algorithmic stability~\cite{Bousquet01,Bousquet02,Lan08,London13} demonstrated that
a stable learning algorithm has the property that replacing one element in the training set
does not result in significant change to the algorithm's output. Therefore, if we take the training error
as a random variable, the training error of stable learning algorithm should have a
small variance. This implies that stable algorithms have the property that the training errors
are close to the test errors~\cite{Bousquet01,Lan08,London13}.
The rest of this section introduces and analyzes the algorithm stability problem of matrix
approximation.

In this section, we adopt Root Mean Square Error (RMSE) 
to measure the stability of matrix approximation as an example. Note that, other kinds of popular errors,
e.g., Mean Square Error (MSE), ``0-1'' loss, etc., can also be applied with only small
changes in the following analysis.
Let $\mathcal{D}(\hat{R}) = \sqrt{\frac{1}{mn}\sum_{i=1}^m\sum_{j=1}^n{(R_{i,j}-\hat{R}_{i,j})^2}}$
be the root mean square error of approximating $R$ with $\hat{R}$.
One of the most popular objectives of matrix approximation is to approximate a given matrix $R$
based on a set of observed entries $\Omega$
($\mathcal{D}_{\Omega}(\hat{R}) = \sqrt{\frac{1}{|\Omega|}\sum_{(i,j)\in\Omega}{(R_{i,j}-\hat{R}_{i,j})^2}}$).
Thus, the stability of approximating $R$ is defined as follows.

\begin{definition}[Stability w.r.t. Matrix Approximation]
\label{def:stability}
For any $R\in \mathbb{F}^{m\times n}$, choose a subset of entries $\Omega$
from $R$ uniformly. For a given $\epsilon>0$, we say that $\mathcal{D}_{\Omega}(\hat{R})$
is $\delta$-stable if the following holds:
\begin{equation}
\Pr[|\mathcal{D}(\hat{R}) - \mathcal{D}_{\Omega}(\hat{R})|<\epsilon]\ge 1-\delta.
\end{equation}
\end{definition}

Matrix approximation with stability guarantee has the property that the
generalization error is bounded. Minimizing the training error will have a high
probability of minimizing the test error.
The stability notion introduced in this work defines how stable an approximation is
in terms of the overall prediction error. It is  different from the Uniform Stability
definition~\cite{Bousquet01}, which defines the prediction stability on individual entries.
This new stability notion makes
it possible to measure the generalization performance between different approximations.
For instance, for any two different approximations of $R$ --- $\hat{R}_1$ and $\hat{R}_2$,
which are $\delta_1$-stable and $\delta_2$-stable, respectively, then
$\mathcal{D}_{\Omega}(\hat{R}_1)$ is more stable than $\mathcal{D}_{\Omega}(\hat{R}_2)$
if $\delta_1<\delta_2$. This implies that $\mathcal{D}_{\Omega}(\hat{R}_1)$ is close to
$\mathcal{D}(\hat{R})$ with higher probability than $\mathcal{D}_{\Omega}(\hat{R}_2)$,
i.e., solving the optimization problem defined by $\mathcal{D}_{\Omega}(\hat{R}_1)$
will lead to solutions that are more likely to have better generalization performance
than that of $\mathcal{D}_{\Omega}(\hat{R}_2)$. In summary, using the stability notion
introduced in this paper, we can compare the stability bounds of different matrix approximation problems
and define new matrix approximation problems which can yield
solutions with high stability, i.e., high generalization performance.

\subsection{Stability vs. Generalization Error}
\begin{figure}[htb!]
  \centerline{\includegraphics[width=0.42\textwidth]{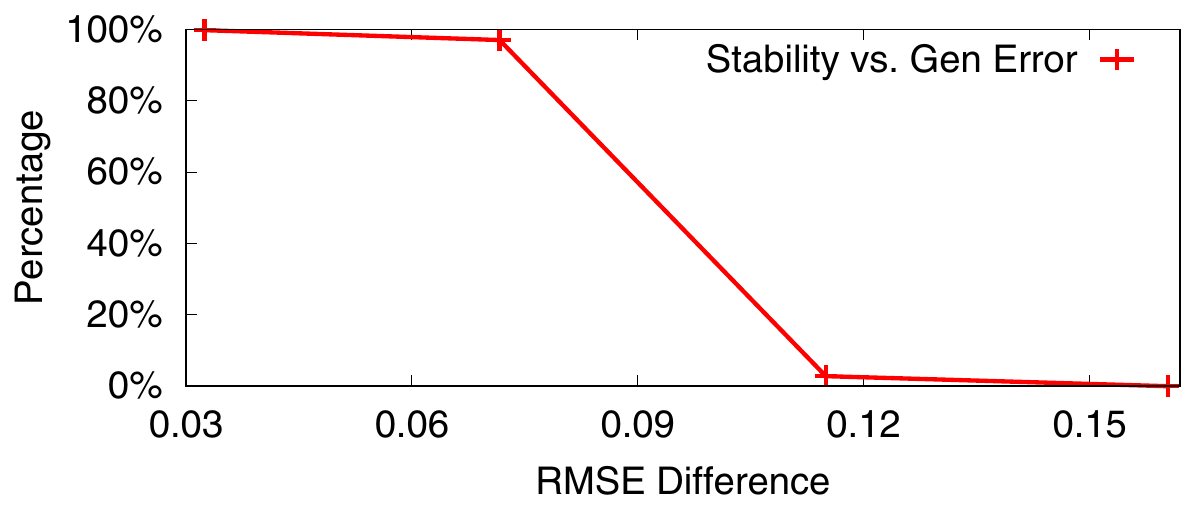}}
  \caption{Stability vs. generalization error with different rank $r$ on the MovieLens (1M) dataset.}
  \label{fig:motivation}
\end{figure}

Figure~\ref{fig:motivation} quantifies stability changes of matrix approximation method
with the generalization error when rank $r$ increases from 5 to 20. This experiment uses
RSVD~\cite{paterek2007improving}, a popular MA-based recommendation algorithm,
on the MovieLens (1M) dataset ({\scriptsize${\sim}$}$10^6$ ratings, 6,040 users, 3,706 items).
We choose $\epsilon$ in Definition~\ref{def:stability} as 0.0046 to
cover all error differences when $r=5$.
We compute $\Pr[|\mathcal{D}(\hat{R}) - \mathcal{D}_{\Omega}(\hat{R})|<\epsilon]$
with 500 different runs to measure stability (y-axis), and compute RMSE differences
between training and test data to measure generalization error (x-axis).

As shown in Figure~\ref{fig:motivation}, the generalization error increases when
rank $r$ increases, because matrix approximation models become more complex
and more biased on training data with higher ranks. In contrary,
the stability of RSVD decreases when $r$ increases. This indicates that
(1) stability decreases when generalization error increases, and
(2) RSVD cannot provide stable recommendation even when the rank is as low as 20.
This study demonstrates that existing matrix approximation methods suffer from low generalization
performance due to low algorithm stability. Therefore, it is important to develop
stable matrix approximation methods with good generalization performance.

\section{Stability of Matrix Approximation}
\label{sec:theory}
In this section, we analyze the stability of matrix approximation in two different
collaborative filtering tasks: 1) rating prediction and 2) top-N recommendation.
We formally prove our key observations that introducing properly selected subsets
into the loss functions can improve the stability of matrix approximation methods.

\subsection{Stability Analysis of Matrix Approximation in the Rating Prediction Task}
We first introduce the Hoeffding's Lemma, and then analyze the
stability properties of low-rank matrix approximation problems.
\begin{lemma}[Hoeffding's Lemma]
\label{lm:hoeffding}
Let $X$ be a real-valued random variable with zero mean and $\Pr\left(X\in \left[b,a\right]\right)=1$.
Then, for any $s\in \mathbb{R}$,
$\mathrm {E} \left[e^{sX}\right]\leq \exp \left({\tfrac {1}{8}}s^{2}(a-b)^{2}\right)$.
\end{lemma}

Following the Uniform Stability~\cite{Bousquet01}, given a stable matrix approximation
algorithm, the approximation results remain stable if the change of the
training data set, i.e., the set of observed entries $\Omega$ from the original matrix
$R$,  is small.
For instance, we can remove a subset of easily predictable entries from $\Omega$ to obtain $\Omega'$.
It is desirable that the solution of minimizing both $\mathcal{D}_\Omega$ and $\mathcal{D}_{\Omega'}$ together
will be more stable than the solution of minimizing $\mathcal{D}_\Omega$ only.
The following Theorem~\ref{thm:stable_1} formally proves the statement.
\begin{theorem}
\label{thm:stable_1}
Let $\Omega$ ($|\Omega|>2$) be a set of observed entries in $R$.
Let $\omega\subset\Omega$ be a subset of observed entries, which satisfies that
$\forall (i,j)\in\omega$, $|R_{i,j}-\hat{R}_{i,j}|\le\mathcal{D}_{\Omega}(\hat{R})$.
Let $\Omega'=\Omega-\omega$, 
then for any $\epsilon>0$ and $1>\lambda_0,\lambda_1>0$ ($\lambda_0+\lambda_1=1$),
$\lambda_0\mathcal{D}_{\Omega}(\hat{R})+ \lambda_1\mathcal{D}_{\Omega'}(\hat{R})$ and
$\mathcal{D}_{\Omega}(\hat{R})$ are $\delta_1$-stable and $\delta_2$-stable, resp.,
then $\delta_1\le\delta_2$.
\end{theorem}
\begin{proof}
Let's assume that $\mathcal{D}(\hat{R}) - \mathcal{D}_{\Omega}(\hat{R})\in[-a, a]$
($a=\sup\{\mathcal{D}(\hat{R}) - \mathcal{D}_{\Omega}(\hat{R})\}$) and
$\mathcal{D}(\hat{R}) - (\lambda_0\mathcal{D}_{\Omega}(\hat{R}) + \lambda_1\mathcal{D}_{\Omega'}(\hat{R}))\in[-a',a']$
($a'=\sup\{\mathcal{D}(\hat{R}) - (\lambda_0\mathcal{D}_{\Omega}(\hat{R}) + \lambda_1\mathcal{D}_{\Omega'}(\hat{R}))\}$)
are two random variables with $0$ mean.

Based on Markov's inequality, for any $t>0$, we have
\begin{equation}
\Pr[\mathcal{D}(\hat{R}) - \mathcal{D}_{\Omega}(\hat{R})\ge\epsilon] \le
\frac{E[\exp{(t(\mathcal{D}(\hat{R}) - \mathcal{D}_{\Omega}(\hat{R})))}]}{\exp{(t\epsilon)}}.\nonumber
\end{equation}
Then, based on Lemma~\ref{lm:hoeffding}, we have
$E[\exp{(t(\mathcal{D}(\hat{R}) - \mathcal{D}_{\Omega}(\hat{R})))}]\le \exp{(\frac{1}{2}t^2a^2)}$, i.e.,
$\Pr[\mathcal{D}(\hat{R}) - \mathcal{D}_{\Omega}(\hat{R})\ge\epsilon] \le
\frac{\exp{(\frac{1}{2}t^2a^2)}}{\exp{(t\epsilon)}}$.
And similarly, we have
$\Pr[\mathcal{D}(\hat{R}) - \mathcal{D}_{\Omega}(\hat{R})\le -\epsilon] \le
\frac{\exp{(\frac{1}{2}t^2a^2)}}{\exp{(t\epsilon)}}$.
Combining the two equations above, we have
$\Pr[|\mathcal{D}(\hat{R}) - \mathcal{D}_{\Omega}(\hat{R})|\ge\epsilon] \le
\frac{2\exp{(\frac{1}{2}t^2a^2)}}{\exp{(t\epsilon)}}$, i.e.,
\begin{equation}
\Pr[|\mathcal{D}(\hat{R}) - \mathcal{D}_{\Omega}(\hat{R})|<\epsilon]
\ge 1 - \frac{2\exp{(\frac{1}{2}t^2a^2)}}{\exp{(t\epsilon)}}.
\end{equation}
Similarly, we have
\begin{equation}
\Pr[|\mathcal{D}(\hat{R}) - (\lambda_0\mathcal{D}_{\Omega}(\hat{R}) + \lambda_1\mathcal{D}_{\Omega'}(\hat{R}))|<\epsilon]
\ge 1 - \frac{2\exp{(\frac{1}{2}t^2a'^2)}}{\exp{(t\epsilon)}}.
\end{equation}
We can compare $a'$ with $a$ as follows:
\begin{align}
a' &= \sup\{\mathcal{D}(\hat{R}) - \mathcal{D}_{\Omega}(\hat{R}) +
\lambda_1(\mathcal{D}_{\Omega}(\hat{R})-\mathcal{D}_{\Omega'}(\hat{R}))\}\nonumber\\
&= \sup\{\mathcal{D}(\hat{R}) - \mathcal{D}_{\Omega}(\hat{R})\} +
\lambda_1\sup\{\mathcal{D}_{\Omega}(\hat{R})-\mathcal{D}_{\Omega'}(\hat{R})\}\nonumber\\
&= a + \lambda_1\sup\{\mathcal{D}_{\Omega}(\hat{R})-\mathcal{D}_{\Omega'}(\hat{R})\}.\nonumber
\end{align}
Since $\forall (i,j)\in\omega$, $|R_{i,j}-\hat{R}_{i,j}|\le\mathcal{D}_{\Omega}(\hat{R})$, we have
$1/|\omega|\sum_{(i,j)\in\omega}(R_{i,j}-\hat{R}_{i,j})^2\le\mathcal{D}_{\Omega}^2(\hat{R})$, i.e.,
$\mathcal{D}_{\omega}(\hat{R})\le\mathcal{D}_{\Omega}(\hat{R})$. Then, since $\Omega=\omega\cup\Omega'$,
we have $\mathcal{D}_{\Omega'}(\hat{R})\ge\mathcal{D}_{\Omega}(\hat{R})$.
This means that $\sup\{\mathcal{D}_{\Omega}(R)-\mathcal{D}_{\Omega'}(R)\}\le0$.
Thus, we have $a'\le a$. Therefore, 
$\frac{2\exp{(\frac{1}{2}t^2a'^2)}}{\exp{(t\epsilon)}} \le \frac{2\exp{(\frac{1}{2}t^2a^2)}}{\exp{(t\epsilon)}}$,
i.e., $\delta_1\le\delta_2$.
\end{proof}

{\bf Remark.}
Theorem~\ref{thm:stable_1} above indicates that, if we remove a subset of entries that
are easier to predict than average from $\Omega$ to form $\Omega'$, then
$\lambda_0\mathcal{D}_{\Omega}(\hat{R})+ \lambda_1\mathcal{D}_{\Omega'}(\hat{R})$
has a higher probability of being close to $\mathcal{D}(\hat{R})$ than $\mathcal{D}_{\Omega}(\hat{R})$.
Therefore, minimizing $\lambda_0\mathcal{D}_{\Omega}(\hat{R})+\lambda_1\mathcal{D}_{\Omega'}(\hat{R})$
will lead to solutions that have better generalization performance than minimizing $\mathcal{D}_{\Omega}(\hat{R})$.
It should be noted that the condition that $\forall (i,j)\in\omega$,
$|R_{i,j}-\hat{R}_{i,j}|\le\mathcal{D}_{\Omega}(\hat{R})$ is not necessary.
The conclusion will be the same if $\mathcal{D}_{\omega}(\hat{R})\le\mathcal{D}_{\Omega}(\hat{R})$.
The following Proposition~\ref{prop:stable_1} formally proves this.
\begin{proposition}
Let $\Omega$ ($|\Omega|>2$) be a set of observed entries in $R$.
Let $\omega\subset\Omega$ be a subset of observed entries, which satisfies that
$D_{\omega}(\hat{R})\le D_{\Omega}(\hat{R})$. Let $\Omega'=\Omega-\omega$,
then for any $\epsilon>0$ and $1>\lambda_0,\lambda_1>0$ ($\lambda_0+\lambda_1=1$),
$\lambda_0\mathcal{D}_{\Omega}(\hat{R})+ \lambda_1\mathcal{D}_{\Omega'}(\hat{R})$ and
$\mathcal{D}_{\Omega}(\hat{R})$ are $\delta_1$-stable and $\delta_2$-stable, resp.,
then $\delta_1\le\delta_2$.
\label{prop:stable_1}
\end{proposition}
\begin{proof}
This proof is omitted as it is similar to that of Theorem~\ref{thm:stable_1}.
\end{proof}
However, Theorem~\ref{thm:stable_1} and Proposition~\ref{prop:stable_1} only
prove that it is beneficial to remove easily predictable entries from $\Omega$
to obtain $\Omega'$, but does not show how many entries should be removed from $\Omega$.
The following Theorem~\ref{thm:stable_2} shows that removing more entries that satisfy
$|R_{i,j}-\hat{R}_{i,j}|\le\mathcal{D}_{\Omega}(\hat{R})$ can yield better $\Omega'$.

\begin{theorem}
\label{thm:stable_2}
Let $\Omega$ ($|\Omega|>2$) be a set of observed entries in $R$.
Let $\omega_2\subset\omega_1\subset\Omega$, and $\omega_1$ and $\omega_2$ satisfy that
$\forall (i,j)\in\omega_1(\omega_2)$, $|R_{i,j}-\hat{R}_{i,j}|\le\mathcal{D}_{\Omega}(\hat{R})$.
Let $\Omega_1 = \Omega - \omega_1$ and $\Omega_2 = \Omega - \omega_2$,
then for any $\epsilon>0$ and $1>\lambda_0,\lambda_1>0$ ($\lambda_0+\lambda_1=1$),
$\lambda_0\mathcal{D}_{\Omega}(\hat{R})+ \lambda_1\mathcal{D}_{\Omega_1}(\hat{R})$ and
$\lambda_0\mathcal{D}_{\Omega}(\hat{R})+ \lambda_1\mathcal{D}_{\Omega_2}(\hat{R})$ are
$\delta_1$-stable and $\delta_2$-stable, resp., then $\delta_1\le\delta_2$.
\end{theorem}
\begin{proof}
Similar to Theorem~\ref{thm:stable_1}, let's assume that
$\mathcal{D}(\hat{R}) - (\lambda_0\mathcal{D}_{\Omega}(\hat{R}) + \lambda_1\mathcal{D}_{\Omega_1}(\hat{R}))\in[-a_1,a_1]$
($a_1=\sup\{\mathcal{D}(\hat{R}) - (\lambda_0\mathcal{D}_{\Omega}(\hat{R}) + \lambda_1\mathcal{D}_{\Omega_1}(\hat{R}))\}$) and
$\mathcal{D}(\hat{R}) - (\lambda_0\mathcal{D}_{\Omega}(\hat{R}) + \lambda_1\mathcal{D}_{\Omega_2}(\hat{R}))\in[-a_2,a_2]$
($a_2=\sup\{\mathcal{D}(\hat{R}) - (\lambda_0\mathcal{D}_{\Omega}(\hat{R}) + \lambda_1\mathcal{D}_{\Omega_2}(\hat{R}))\}$)
are two random variables with 0 mean.

Applying Lemma~\ref{lm:hoeffding} and the Markov's inequality, we have
\begin{eqnarray}
\Pr[|\mathcal{D}(\hat{R}) - (\lambda_0\mathcal{D}_{\Omega}(\hat{R}) + \lambda_1\mathcal{D}_{\Omega_1}(\hat{R}))|<\epsilon]
&\ge& 1 - \frac{2\exp{(\frac{1}{2}t^2a_1^2)}}{\exp{(t\epsilon)}} \nonumber \\
\Pr[|\mathcal{D}(\hat{R}) - (\lambda_0\mathcal{D}_{\Omega}(\hat{R}) + \lambda_1\mathcal{D}_{\Omega_2}(\hat{R}))|<\epsilon]
&\ge& 1 - \frac{2\exp{(\frac{1}{2}t^2a_2^2)}}{\exp{(t\epsilon)}}. \nonumber
\end{eqnarray}
Since $\forall (i,j)\in\omega_1(\omega_2)$, $|R_{i,j}-\hat{R}_{i,j}|\le\mathcal{D}_{\Omega}(\hat{R})$
and $\omega_2\subset\omega_1$, we have $\mathcal{D}_{\Omega_1}(\hat{R})\le \mathcal{D}_{\Omega_2}(\hat{R})$.
Thus, we have $\sup\{\mathcal{D}_{\Omega_1}(\hat{R})-\mathcal{D}_{\Omega_2}(\hat{R})\}\le0$.
Since $a_1 = a_2 + \lambda_1\sup\{\mathcal{D}_{\Omega_1}(\hat{R})-\mathcal{D}_{\Omega_2}(\hat{R})\}$,
we have $a_1\le a_2$. Then, similar to Theorem~\ref{thm:stable_1}, we can conclude that
$\delta_1\le\delta_2$.
\end{proof}

{\bf Remark.}
From Theorem~\ref{thm:stable_2}, we know that removing more entries that are
easy to predict will yield more stable matrix approximation. Therefore, it is
desirable to choose $\Omega'$ as the whole set of entries which are harder to
predict than average, i.e., the whole set of entries satisfying
$\forall (i,j)\in\Omega'$, $|R_{i,j}-\hat{R}_{i,j}|\ge\mathcal{D}_{\Omega}(\hat{R})$.

Theorem~\ref{thm:stable_1} and~\ref{thm:stable_2} only consider choosing one $\Omega'$
to find stable matrix approximations. The next question is, is it possible to choose
more than one $\Omega'$ that satisfy the above condition, and yield
more stable solutions by minimizing them all together.
The following Theorem~\ref{thm:stable_3} shows that incorporating $K$ such entry sets ($\Omega'$)
will be more stable than incorporating any $K-1$ out of the $K$ entry sets.

\begin{theorem}
\label{thm:stable_3}
Let $\Omega$ ($|\Omega|>2$) be a set of observed entries in $R$.
$\omega_1,...,\omega_K \subset \Omega$ ($K>1$) satisfy that
$\forall (i,j)\in\omega_k$ ($1\le k\le K$), $|R_{i,j}-\hat{R}_{i,j}|\le\mathcal{D}_{\Omega}(\hat{R})$.
Let $\Omega_k = \Omega -\omega_k$ for all $1\le k\le K$. Then, for any $\epsilon>0$
and $1>\lambda_0,\lambda_1,...,\lambda_K>0$ ($\sum_{i=0}^K{\lambda_i}=1$),
$\lambda_0\mathcal{D}_{\Omega}(\hat{R})+ \sum_{k\in[1,K]}{\lambda_k\mathcal{D}_{\Omega_k}(\hat{R})}$
and $(\lambda_0+\lambda_K)\mathcal{D}_{\Omega}(\hat{R})
+ \sum_{k\in[1,K-1]}{\lambda_k\mathcal{D}_{\Omega_k}(\hat{R})}$ are
$\delta_1$-stable and $\delta_2$-stable, resp., then $\delta_1\le\delta_2$.
\end{theorem}
\begin{proof}
Assume that $\mathcal{D}(\hat{R}) - ((\lambda_0+\lambda_K)\mathcal{D}_{\Omega}(\hat{R})
+ \sum_{k\in[1,K-1]}{\lambda_k\mathcal{D}_{\Omega_k}(\hat{R})})\in[-a, a]$
($a=\sup\{\mathcal{D}(\hat{R}) -
((\lambda_0+\lambda_K)\mathcal{D}_{\Omega}(\hat{R})
+ \sum_{k\in[1,K-1]}{\lambda_k\mathcal{D}_{\Omega_k}(\hat{R})})\}$ and
$\mathcal{D}(\hat{R}) -
(\lambda_0\mathcal{D}_{\Omega}(\hat{R})+ \sum_{k\in[1,K]}{\lambda_k\mathcal{D}_{\Omega_k}(\hat{R})})\in[-a',a']$
($a'=\sup\{\mathcal{D}(\hat{R}) -
(\lambda_0\mathcal{D}_{\Omega}(\hat{R})+ \sum_{k\in[1,K]}{\lambda_k\mathcal{D}_{\Omega_k}(\hat{R})})\}$
are two random variables with $0$ mean.

Applying Lemma~\ref{lm:hoeffding} and the Markov's inequality, we have
\begin{eqnarray}
\Pr[|\mathcal{D}(\hat{R}) -
(\lambda_0\mathcal{D}_{\Omega}(\hat{R})+ \sum_{k\in[1,K]}{\lambda_k\mathcal{D}_{\Omega_k}(\hat{R})})|<\epsilon]
&\ge& 1 - \frac{2\exp{(\frac{1}{2}t^2a'^2)}}{\exp{(t\epsilon)}} \nonumber \\
\Pr[|\mathcal{D}(\hat{R}) - ((\lambda_0+\lambda_K)\mathcal{D}_{\Omega}(\hat{R})
+ \sum_{k\in[1,K-1]}{\lambda_k\mathcal{D}_{\Omega_k}(\hat{R})})|<\epsilon]
&\ge& 1 - \frac{2\exp{(\frac{1}{2}t^2a^2)}}{\exp{(t\epsilon)}}. \nonumber
\end{eqnarray}
Similar as the proofs above, we have $\mathcal{D}_{\Omega_K}(\hat{R})\ge\mathcal{D}_\Omega(\hat{R})$,
which means $\sup\{\mathcal{D}_{\Omega}(\hat{R})-\mathcal{D}_{\Omega_K}(\hat{R})\}\le0$.
Since $a' = a + \lambda_K\sup\{\mathcal{D}_{\Omega}(\hat{R})-\mathcal{D}_{\Omega_K}(\hat{R})\}$,
we know that $a'\le a$.
Thus, we can conclude that $\frac{2\exp{(\frac{1}{2}t^2a'^2)}}{\exp{(t\epsilon)}}\le\frac{2\exp{(\frac{1}{2}t^2a^2)}}{\exp{(t\epsilon)}}$, i.e.,
$\delta_1\le\delta_2$.
\end{proof}

{\bf Remark.}
Theorem~\ref{thm:stable_3} shows that minimizing $\mathcal{D}_\Omega$ together with
the RMSEs of more than one hard predictable subsets of $\Omega$  will
help generate more stable matrix approximation solutions.

However, sometimes it is expensive to find such ``hard predictable subsets'',
because we do not know which subset of entries to choose without any prior knowledge.
Thus, we propose a solution to obtain hard predictable subsets of $\Omega$
based on only one set of easily predictable entries as follows:
1) choose $\omega\subset\Omega$, which satisfies that
$\forall (i,j)\in\omega$, $|R_{i,j}-\hat{R}_{i,j}|\le\mathcal{D}_{\Omega}(\hat{R})$,
and choose $\Omega_0 = \Omega-\omega$;
2) divide $\omega$ into $K$ non-overlapping subsets $\omega_1,...,\omega_K$ with
the condition that $\cup_{k\in[1,K]}\omega_k = \omega$, and
choose $\Omega_k = \Omega - \omega_k$ for all $1\le k\le K$; and
3) minimize $\lambda_0\mathcal{D}_\Omega(\hat{R})+\sum_{k=1}^K{\lambda_k\mathcal{D}_{\Omega_k}(\hat{R})}$
to find stable matrix approximation solutions.
The following Theorem~\ref{thm:stable_4} proves that it is desirable to minimize
$\lambda_0\mathcal{D}_\Omega(\hat{R})+\sum_{k=1}^K{\lambda_k\mathcal{D}_{\Omega_k}(\hat{R})}$
instead of $\lambda_0\mathcal{D}_\Omega(\hat{R})+(1-\lambda_0)\mathcal{D}_{\Omega_0}(\hat{R})$.

\begin{theorem}
\label{thm:stable_4}
Let $\Omega$ ($|\Omega|>2$) be a set of observed entries in $R$.
Choose $\omega\subset\Omega$, which satisfies that
$\forall (i,j)\in\omega$, $|R_{i,j}-\hat{R}_{i,j}|\le\mathcal{D}_{\Omega}(\hat{R})$.
And divide $\omega$ into $K$ non-overlapping subsets $\omega_1,...,\omega_K$ with
the condition that $\cup_{k\in[1,K]}\omega_k = \omega$.
Let $\Omega_0 = \Omega-\omega$ and $\Omega_k = \Omega - \omega_k$ for all $1\le k\le K$.
Then, for any $\epsilon>0$ and $1>\lambda_0,\lambda_1,...,\lambda_K>0$ ($\sum_{i=0}^K{\lambda_i}=1$),
$\lambda_0\mathcal{D}_\Omega(\hat{R})+\sum_{k=1}^K{\lambda_k\mathcal{D}_{\Omega_k}(\hat{R})}$ and
$\lambda_0\mathcal{D}_\Omega(\hat{R})+(1-\lambda_0)\mathcal{D}_{\Omega_0}(\hat{R})$ are $\delta_1$-stable
and $\delta_2$-stable, resp., then $\delta_1\le\delta_2$.
\end{theorem}
\begin{proof}
Let's first assume that $\mathcal{D}(\hat{R}) -
(\lambda_0\mathcal{D}_\Omega(\hat{R})+\sum_{k=1}^K{\lambda_k\mathcal{D}_{\Omega_k}(\hat{R})})\in[-a_1, a_1]$
($a_1=\sup\{\mathcal{D}(\hat{R}) -
(\lambda_0\mathcal{D}_\Omega(\hat{R})+\sum_{k=1}^K{\lambda_k\mathcal{D}_{\Omega_k}(\hat{R})})\}$ and
$\mathcal{D}(\hat{R}) -
(\lambda_0\mathcal{D}_\Omega(\hat{R})+(1-\lambda_0)\mathcal{D}_{\Omega_0}(\hat{R}))\in[-a_2, a_2]$
($a_2=\sup\{\mathcal{D}(\hat{R}) - ((1-\lambda_0)\mathcal{D}_{\Omega_0}(\hat{R}))\}$
are two random variables with 0 mean.

We have
$\Pr[|\mathcal{D}(\hat{R}) -
(\lambda_0\mathcal{D}_{\Omega}(\hat{R})+ \sum_{k=1}^K{\lambda_k\mathcal{D}_{\Omega_k}(\hat{R})})|<\epsilon]
\ge 1 - \frac{2\exp{(\frac{1}{2}t^2a_1^2)}}{\exp{(t\epsilon)}}$ and
$\Pr[|\mathcal{D}(\hat{R}) - ((1-\lambda_0)\mathcal{D}_{\Omega_0}(\hat{R}))|<\epsilon]
\ge 1 - \frac{2\exp{(\frac{1}{2}t^2a_2^2)}}{\exp{(t\epsilon)}}$. $\forall k\in[1,K]$
$\omega_k\subset\omega$ and $\forall (i,j)\in\omega$, $|R_{i,j}-\hat{R}_{i,j}|\le\mathcal{D}_{\Omega}(\hat{R})$,
we have for all $k\in[1,K]$, $\mathcal{D}_{\Omega_k}\le\mathcal{D}_{\Omega_0}$.
Sum the above inequation over all $k\in[1,K]$, we have
$\sum_{k=1}^K\lambda_k\mathcal{D}_{\Omega_k}\le\sum_{k=1}^K\lambda_k\mathcal{D}_{\Omega_0}
=(1-\lambda_0)\mathcal{D}_{\Omega_0}$.
Thus, $\sup\{\mathcal{D}(\hat{R}) -
(\lambda_0\mathcal{D}_\Omega(\hat{R})+\sum_{k=1}^K{\lambda_k\mathcal{D}_{\Omega_k}(\hat{R})})\}
\le
\sup\{\mathcal{D}(\hat{R}) - ((1-\lambda_0)\mathcal{D}_{\Omega_0}(\hat{R}))\}$, i.e., $a_1\le a_2$.
Thus, we can conclude that $\delta_1\le\delta_2$.
\end{proof}

{\bf Remark.}
Theorem~\ref{thm:stable_4} shows that if we can find only one subset of entries
that are easier to predict than average, then we can probe this subset of entries
to increase the stability of matrix approximations.

\subsection{Stability Analysis of Matrix Approximation in the General Setting}

So far, all theoretical analysis rely on one key assumption:
the value of $\mathcal{D}(\hat{R})-\mathcal{D}_\Omega(\hat{R})$
is within a symmetric interval, e.g., $[-a, a]$ ($a=\sup\{\mathcal{D}(\hat{R})-\mathcal{D}_\Omega(\hat{R})\}$).
The assumption holds for matrix approximation in the rating prediction tasks,
but may not hold for all kinds of matrix approximation problems,
e.g., binary matrix approximation for top-N recommendation.
For instance, binary matrix approximation typically minimizes ``0-1'' loss which is
non-convex,  so $\mathcal{D}_\Omega(\hat{R})$
is typically defined as convex surrogate losses~\cite{Cofi}.
Apparently, the infimum and supremum of convex surrogated $\mathcal{D}(\hat{R})-\mathcal{D}_\Omega(\hat{R})$
are not always symmetric around 0, e.g., the Exponential loss and Log loss.

For a more general case which holds for all kinds of matrix approximation problems,
we only assume that $E[\mathcal{D}(\hat{R})-\mathcal{D}_\Omega(\hat{R})] = 0$,
$\mathcal{D}(\hat{R})-\mathcal{D}_\Omega(\hat{R})\in [b, a]$,
where $a=\sup\{\mathcal{D}(\hat{R})-\mathcal{D}_\Omega(\hat{R})\}$ and $b=\inf\{\mathcal{D}(\hat{R})-\mathcal{D}_\Omega(\hat{R})\}$.
The following theorem proves that we can derive matrix approximation problems with
higher stability by randomly choosing $\Omega'\subset\Omega$
with the above general assumption.

\begin{theorem}
\label{thm:stable_5}
Let $\Omega$ ($|\Omega|>2$) be a set of observed entries in $R$ and
$\Omega'\subset\Omega$ be a randomly selected subset of observed entries.
Assume that, for any $1>\lambda_0,\lambda_1>0$ ($\lambda_0+\lambda_1=1$),
$\lambda_0\mathcal{D}_{\Omega}(\hat{R})+ \lambda_1\mathcal{D}_{\Omega'}(\hat{R})$ and
$\mathcal{D}_{\Omega}(\hat{R})$ are $\delta_1$-stable and $\delta_2$-stable, resp.,
then $\delta_1\le\delta_2$.
\end{theorem}
\begin{proof}
Based on Markov's inequality and Lemma~\ref{lm:hoeffding}, for any $t>0$ and $\epsilon>0$, we have
\begin{equation}
\Pr[\mathcal{D}(\hat{R}) - \mathcal{D}_{\Omega}(\hat{R})\ge\epsilon] \le
\frac{E[e^{t(\mathcal{D}(\hat{R}) - \mathcal{D}_{\Omega}(\hat{R}))}]}{e^{t\epsilon}}
\le \frac{e^{\frac{1}{8}t^2(a-b)^2}}{e^{t\epsilon}}, \nonumber
\end{equation}
where $a=\sup\{\mathcal{D}(\hat{R})-\mathcal{D}_\Omega(\hat{R})\}$ and $b=\inf\{\mathcal{D}(\hat{R})-\mathcal{D}_\Omega(\hat{R})\}$.

Based on Markov's inequality and the convexity of exponential function, we have
\begin{eqnarray}
\Pr[\mathcal{D}(\hat{R})-\lambda_0\mathcal{D}_\Omega(\hat{R}) - \lambda_1\mathcal{D}_{\Omega'}(\hat{R})\ge\epsilon] &\le&
\frac{E[e^{t(\mathcal{D}(\hat{R})-\lambda_0\mathcal{D}_\Omega(\hat{R}) - \lambda_1\mathcal{D}_{\Omega'}(\hat{R}))}]}{e^{t\epsilon}} \nonumber\\
&=& \frac{E[e^{t(\lambda_0(\mathcal{D}(\hat{R})-\mathcal{D}_\Omega(\hat{R})) +
\lambda_1(\mathcal{D}(\hat{R})-\mathcal{D}_{\Omega'}(\hat{R})))}]}{e^{t\epsilon}} \nonumber \\
&\le& \frac{ \lambda_0E[e^{t(\mathcal{D}(\hat{R}) - \mathcal{D}_{\Omega}(\hat{R}))}]
+\lambda_1E[e^{t(\mathcal{D}(\hat{R}) - \mathcal{D}_{\Omega'}(\hat{R}))}] } {e^{t\epsilon}} \nonumber \\
&\le& \lambda_0\frac{e^{\frac{1}{8}t^2(a-b)^2}}{e^{t\epsilon}} + \lambda_1\frac{e^{\frac{1}{8}t^2(a'-b')^2}}{e^{t\epsilon}}. \nonumber
\end{eqnarray}
where $a$ and $b$ are the same as above and
$a'=\sup\{\mathcal{D}(\hat{R})-\mathcal{D}_{\Omega'}(\hat{R})\}$ and $b'=\inf\{\mathcal{D}(\hat{R})-\mathcal{D}_{\Omega'}(\hat{R})\}$.

Based on the properties of infimum and supremum, we have
\begin{eqnarray}
(a-b)^2 &=& ((\sup\{\mathcal{D}(\hat{R})\} - \inf\{\mathcal{D}(\hat{R})\}) +
    (\sup\{\mathcal{D}_{\Omega}(\hat{R})\} - \inf\{\mathcal{D}_{\Omega}(\hat{R})\}) )^2 \nonumber \\
(a'-b')^2 &=& ((\sup\{\mathcal{D}(\hat{R})\} - \inf\{\mathcal{D}(\hat{R})\}) +
    (\sup\{\mathcal{D}_{\Omega'}(\hat{R})\} - \inf\{\mathcal{D}_{\Omega'}(\hat{R})\}) )^2 \nonumber
\end{eqnarray}
Since $\Omega'\subset\Omega$, we know that
$\sup\{\mathcal{D}_{\Omega'}(\hat{R})\}\le \sup\{\mathcal{D}_{\Omega}(\hat{R})\}$ and
$\inf\{\mathcal{D}_{\Omega'}(\hat{R})\}\ge \inf\{\mathcal{D}_{\Omega}(\hat{R})\}$,
so that $0\le\sup\{\mathcal{D}_{\Omega'}(\hat{R})\} - \inf\{\mathcal{D}_{\Omega'}(\hat{R})\} \le \sup\{\mathcal{D}_{\Omega}(\hat{R})\} - \inf\{\mathcal{D}_{\Omega}(\hat{R})\}$. This will turn out that $(a'-b')^2\le(a-b)^2$, which means
\begin{equation}
\lambda_0\frac{e^{\frac{1}{8}t^2(a-b)^2}}{e^{t\epsilon}} + \lambda_1\frac{e^{\frac{1}{8}t^2(a'-b')^2}}{e^{t\epsilon}}
\le \frac{e^{\frac{1}{8}t^2(a-b)^2}}{e^{t\epsilon}}. \nonumber
\end{equation}
Based on the above inequality, we can conclude that $\delta_1\le\delta_2$.
\end{proof}

{\bf Remark.}
Theorem~\ref{thm:stable_5} above indicates that, if we randomly choose
a subset of entries from $\Omega$ to form $\Omega'$, then
$\lambda_0\mathcal{D}_{\Omega}(\hat{R})+ \lambda_1\mathcal{D}_{\Omega'}(\hat{R})$ will also
have a higher probability of being close to $\mathcal{D}(\hat{R})$ than $\mathcal{D}_{\Omega}(\hat{R})$.
Therefore, minimizing $\lambda_0\mathcal{D}_{\Omega}(\hat{R})+\lambda_1\mathcal{D}_{\Omega'}(\hat{R})$
is more desirable than minimizing $\mathcal{D}_{\Omega}(\hat{R})$. Note that,
as shown in the proof, the sharpness of the stability bound depends on
$\sup\{\mathcal{D}_{\Omega'}(\hat{R})\}$ and $\inf\{\mathcal{D}_{\Omega'}(\hat{R})\}$.
However, it is non-trivial to directly infer $\sup\{\mathcal{D}_{\Omega'}(\hat{R})\}$
and $\inf\{\mathcal{D}_{\Omega'}(\hat{R})\}$ from a randomly chosen $\Omega'$.
In this case, $\Omega'$ should be considered as a tunable parameter in practice.
A better choice of $\Omega'$ will yield sharper
stability bound and thus yield better model performance.

Similarly, we can also select subsets from $\Omega'$ and form optimization problems
that are even more stable than $\lambda_0\mathcal{D}_{\Omega}(\hat{R})+ \lambda_1\mathcal{D}_{\Omega'}(\hat{R})$.
The following theorems formally prove this idea.

\begin{theorem}
\label{thm:stable_6}
Let $\Omega$ ($|\Omega|>2$) be a set of observed entries in $R$.
Let $\Omega_1\subset\Omega_2\subset ... \subset\Omega_K\subset\Omega$ ($K>1$)
be randomly selected subsets of observed entries.
Assume that, for any $1>\lambda_0,\lambda_1,...,\lambda_K>0$ ($\lambda_0+\sum_{s=1}^K\lambda_s=1$),
$\lambda_0\mathcal{D}_{\Omega}(\hat{R})+ \sum_{s=1}^K\lambda_s\mathcal{D}_{\Omega_s}(\hat{R})$ and
$\mathcal{D}_{\Omega}(\hat{R})$ are $\delta_1$-stable and $\delta_2$-stable, resp.,
then $\delta_1\le\delta_2$.
\end{theorem}
\begin{proof}
This proof is omitted as it is similar to that of Theorem~\ref{thm:stable_5}.
\end{proof}

Similarly, we can also select $K$ different random subsets to further improve the
stability of matrix approximation, which is formally described in the following Theorem~\ref{thm:stable_7}.
\begin{theorem}
\label{thm:stable_7}
Let $\Omega$ ($|\Omega|>2$) be a set of observed entries in $R$.
Let $\Omega_1\, \Omega_2, ... , \Omega_K$ ($K>1$) be randomly selected subsets of $\Omega$..
Assume that, for any $1>\lambda_0,\lambda_1,...,\lambda_K>0$ ($\lambda_0+\sum_{s=1}^K\lambda_s=1$),
$\lambda_0\mathcal{D}_{\Omega}(\hat{R})+ \sum_{s=1}^K\lambda_s\mathcal{D}_{\Omega_s}(\hat{R})$ and
$\mathcal{D}_{\Omega}(\hat{R})$ are $\delta_1$-stable and $\delta_2$-stable, resp.,
then $\delta_1\le\delta_2$.
\end{theorem}
\begin{proof}
This proof is omitted as it is similar to that of Theorem~\ref{thm:stable_5}.
\end{proof}

{\bf Remark.}
Theorem~\ref{thm:stable_6} and Theorem~\ref{thm:stable_7} above indicate that,
if we randomly choose different subsets of entries from $\Omega$, then we can obtain
even more stable optimization problems than only choosing one $\Omega'$.
The performance of adopting Theorem~\ref{thm:stable_6} will be influenced by
the selection of the largest subset ($\Omega_K$) because all the other subsets depend on it.
Therefore, a bad selection of $\Omega_K$ will not significantly improve the stability.
But the selected subsets are independent in Theorem~\ref{thm:stable_7}. As such,
a few bad selections will not significantly affect the performance. Therefore, it will be more desirable
to adopt Theorem~\ref{thm:stable_7} than Theorem~\ref{thm:stable_6} if
there is no clear clue to select proper subsets.

\section{Stable Matrix Approximation Algorithm for Collaborative Filtering}
\label{sec:algorithm}
This section first presents the stable matrix approximation (SMA) method for
the rating prediction task, and then presents stable matrix approximation
for the top-N recommendation task.
Finally, we present how to solve the algorithm stability optimization
problem using a stochastic gradient descent method.

\subsection{SMA for Rating Prediction}
In the rating prediction task, the goal of stable matrix approximation-based
collaborative filtering is to  minimize the root mean square error (RMSE),
i.e., $\mathcal{D}_{\Omega}(\hat{R}) = \sqrt{\frac{1}{|\Omega|}\sum_{(i,j)\in\Omega}{(R_{i,j}-\hat{R}_{i,j})^2}}$.
Here, we present the stable matrix approximation algorithm in terms of RMSE.

\subsubsection{Model Formulation}
Singular value decomposition (SVD) is one of the commonly used methods for
matrix approximation in the rating prediction task~\cite{Candes10}.
Based on the analysis of stable matrix approximation described in the
previous section, it is desirable to minimize the loss functions that
will lead to solutions with good generalization performance. Let $\{\Omega_1,...,\Omega_K\}$
be subsets of $\Omega$ which satisfy that $\forall s\in[1,K]$, $\mathcal{D}_{\Omega_s}\ge\mathcal{D}_{\Omega}$.
Then, following Theorem~\ref{thm:stable_4},
we propose a new extension of SVD. Note that, extensions to other matrix approximation methods
can be similarly derived.
\begin{align}
&\hat{R} = \argmin_{X}~\lambda_0\mathcal{D}_{\Omega}(X) + \sum_{s=1}^K{\lambda_s\mathcal{D}_{\Omega_s}(X)} \nonumber\\
&s.t.~~ rank(X)=r. \label{eqn:lossFunctnByFro}
\end{align}
where $\lambda_0, \lambda_1,..., \lambda_K$ define the contributions
of each component in the loss function.

\subsubsection{Hard Predictable Subsets Selection}

The key step in Equation~\ref{eqn:lossFunctnByFro} 
is to obtain subsets of $\Omega$ --- $\{\Omega_1,...,\Omega_K\}$ which satisfy
that $\forall s\in[1,K]$, $\mathcal{D}_{\Omega_s}\ge\mathcal{D}_{\Omega}$.
To obtain such $\Omega_s$ is not trivial, because we can only check if
the condition is satisfied with the final model. But the final model cannot
be known before we define and optimize a given loss function.

First, we propose the following idea to address this issue:
1) approximate the target matrix $R$ with existing matrix approximation solutions, e.g., RSVD~\cite{paterek2007improving};
2) choose each entry $(i,j)\in\Omega$ with high probability if
$|R_{i,j}-\hat{R}_{i,j}|<\mathcal{D}_\Omega$ and small probability otherwise; and
3) obtain $\Omega'$ by removing the chosen entries to satisfy the condition of
Proposition~\ref{prop:stable_1}, or probe $\Omega'$ to find hard predictable subsets
that satisfy the condition of Theorem~\ref{thm:stable_4}.
By assuming that other matrix approximation methods will not dramatically differ from
the final model of SMA, we can ensure that $\Omega'$ will satisfy
$\mathcal{D}_{\Omega'}\ge\mathcal{D}_{\Omega}$ with high probability.

\subsection{SMA Algorithm for Top-N Recommendation}
\label{sec:algorithm_topn}
In the top-N recommendation task, the goal of stable matrix approximation
is to minimize the ``0-1'' error as in classification tasks~\cite{Srebro04},
i.e., $\mathcal{D}_{\Omega} = \frac{1}{|\Omega|}\sum_{(i,j)\in\Omega}{\mathbbm{1}(R_{i,j},\hat{R}_{i,j})}$
where $\mathbbm{1}(x,y)$ is an indicator function. $\mathbbm{1}(x,y) = 1$ if $x=y$,
and $\mathbbm{1}(x,y) = 0$ otherwise.
Here, we present the stable matrix approximation algorithm based on the
above ``0-1'' loss.

\subsubsection{Loss Functions}
Since the original ``0-1'' loss is not differentiate, gradient-based
learning algorithms, e.g., SGD, cannot be applied to learn the models.
Therefore, surrogate loss functions are commonly adopted to define convex losses.
The following popular surrogate loss functions~\cite{nguyen2009surrogate,Lee14} are adopted in this paper:
1) mean square error (MSE): $L_{MSE}(\hat{R}_{i,j}, R_{i,j}) = (\hat{R}_{i,j} - R_{i,j})^2$;
2)  log loss (Log): $L_{Log}(\hat{R}_{i,j}, R_{i,j}) = \log(1+\exp\{-\hat{R}_{i,j}R_{i,j}\})$;
and 3) exponential loss (Exp): $L_{Exp}(\hat{R}_{i,j}, R_{i,j}) = \exp\{-\hat{R}_{i,j}R_{i,j}\}$.
After introducing surrogate loss functions, the new loss function of SMA for top-N recommendation
can be defined as follows:
\begin{equation}
\mathcal{D}_{\Omega} = \frac{1}{|\Omega|}\sum_{(i,j)\in\Omega}{L_{*}(R_{i,j},\hat{R}_{i,j})}
\end{equation}
where $L_{*}$ can be any of the surrogate loss functions. The new loss function
is convex because the surrogate loss functions are convex, so that it can be
optimized using gradient-based learning algorithms such as SGD.

In real-world top-N recommendation tasks, the positive ratings are very limited,
which causes severe data imbalance issue. To address this issue, weighting is
adopted by many matrix approximation methods~\cite{Hu08,hsieh2015pu} as follows:
\begin{equation}
\label{eqn:weighted_loss}
\mathcal{D}_{\Omega} = \frac{1}{|\Omega|}\sum_{(i,j)\in\Omega}{W_{i,j}L_{*}(R_{i,j},\hat{R}_{i,j})}
\end{equation}
where $W_{i,j}$ is the weight for entry $(i,j)$. Typically, $W_{i,j}$ is set
to a large value when $R_{i,j}=1$ and a small value when $R_{i,j}=-1$. In this paper,
we fix the weights for positive and negative examples for simplicity of analysis,
i.e., we use a universal weight for all positive examples and another universal
weight for all negative examples.
Recent works~\cite{Hu08,hsieh2015pu} have proved that weighed matrix approximation
methods can improve accuracy of top-N recommendation on implicit feedback data.
Similarly, we can extend matrix approximation with the weighted loss above to generate
a stable matrix approximation method as in Equation~\ref{eqn:lossFunctnByFro}.

\subsubsection{Subset Selection}

Similar to the rating prediction task, subsets of $\Omega$ should be selected to
improve the stability of learned MA models for top-N recommendation.
According to Theorem~\ref{thm:stable_5}, we can randomly choose $\Omega'$ and
optimize the following problem with better stability:
\begin{align}
&\hat{R} = \argmin_{X}~\lambda_0\mathcal{D}_{\Omega}(X) + \lambda_1\mathcal{D}_{\Omega'}(X) \nonumber\\
&s.t.~~ rank(X)=r. \label{eqn:loss_topn_rank}
\end{align}
However, a random $\Omega'$ could improve stability, but it may also hurt optimization
accuracy (training accuracy). Therefore, the test accuracy may not be improved.
To further improve the stability of learned models without much degradation in
optimization accuracy, it is more desirable to select $\Omega'$ as the hard-predictable
subsets, because hard-predictable subsets have tighter gap between infimum and supremum,
i.e., can achieve smaller $\delta$ in the stability measure. For top-N recommendation,
we can assume that examples near the decision boundary are harder to predict than others.
Then, we can select $\Omega'$ as follows:
\begin{equation}
\label{eqn:subset_topn}
\Omega' = \{(i,j) : \hat{R}_{i,j}\in[-\gamma, \gamma]\}.
\end{equation}
where $\gamma>0$ defines the margin to select the examples. The above definition means
that entry $(i,j)$ will be selected in $\Omega'$ if its predicted rating is near 0.
Then, for the $t$-th epoch, we can easily select $\Omega'$ based on the model output
at the $t-1$-th epoch.

\subsubsection{Discussion}

In SMA for top-N recommendation, we change the original ``0-1'' loss to
surrogate losses and adopt weighting to give higher weights to rare
positive examples. Here, we show that the stability analysis from Theorem~\ref{thm:stable_5}
to Theorem~\ref{thm:stable_7} can still be applied to the surrogated and weighted
loss functions. Assuming that the entries in $\Omega$ are independently and identically
selected from its distribution, we can know that the ratio between positive examples
and negative examples is the same for $\Omega$ and the whole set of examples. Then,
we can know that $E[\mathcal{D}(\hat{R}) - \mathcal{D}_{\Omega}(\hat{R})] = 0$,
which means that the precondition of Theorem~\ref{thm:stable_5}
to Theorem~\ref{thm:stable_7} holds. Therefore, we can apply the stability analysis
in these theorems in SMA for top-N recommendation task as described above.

\subsection{The SMA Learning Algorithm}
\begin{algorithm}[tbh!]
\caption{The SMA Learning Algorithm}
\label{alg:sma}
\begin{algorithmic}[1]
\REQUIRE {
$R$ is the targeted matrix, $\Omega$ is the set of entries in $R$,
and $\hat{R}$ is an approximation of $R$ by existing matrix approximation methods.
$p>0.5$ is the predefined probability for entry selection.
$\mu_1$ and $\mu_2$ are the coefficients for L2-regularization.}
\STATE $\Omega' = \emptyset$;
\FOR {each $(i,j)\in\Omega$}
    \STATE randomly generate $\rho\in[0,1]$;
    \IF{($|R_{i,j}-\hat{R}_{i,j}|\le\mathcal{D}_\Omega$ \& $\rho\le p$) or
        ($|R_{i,j}-\hat{R}_{i,j}|>\mathcal{D}_\Omega$ \& $\rho\le 1-p$) }
        \STATE $\Omega' \gets \Omega'\cup\{(i,j)\}$;
    \ENDIF
\ENDFOR
\STATE randomly divide $\Omega'$ into $\omega_1,...,\omega_K$ ($\cup_{k=1}^K\omega_i=\Omega'$);
\STATE for all $k\in[1,K]$, $\Omega_k = \Omega - \omega_k$;
\STATE $(\hat{U},\hat{V}) \eql \argmin_{U, V} [\sum_{k=1}^K{\lambda_k\mathcal{D}_{\Omega_k}(U^{T}V)}$
    $+\lambda_0\mathcal{D}_{\Omega}(UV^{T})
    +\mu_1 \parallel {U} \parallel^{2} + \mu_2 \parallel {V} \parallel^{2}]$
\STATE return $\hat{R} = \hat{U}\hat{V}^{T}$
\end{algorithmic}
\end{algorithm}

Algorithm~\ref{alg:sma}. shows the pseudo-code of the proposed SMA learning
algorithm to solve
the optimization problem defined in Equation~\ref{eqn:lossFunctnByFro}
by adopting hard predictable subsets.
From Step 1 to 9, we obtain $K$ different hard predictable entry sets.
Note that, for top-N recommendation, hard predictable subsets can be
selected by Equation~\ref{eqn:subset_topn}.
In Step 10, the optimization is performed by
stochastic gradient descent (SGD), the details of which are trivial
and thus omitted here. Also, $L_2$ regularization is adopted in Step 10.
Note that, other types of optimization methods and regularization techniques
can also be used in Algorithm~\ref{alg:sma}.
The complexity of Step 1 to 9 is $O(|\Omega|)$ for the rating prediction task and
$O(mn)$ for the top-N recommendation task, where $|\Omega|$
is the number of observed entries in $R$ and $m,n$ is the matrix size.
The complexity of Step 10 is $O(rmn)$ per-iteration,
where $r$ is the rank and $m,n$ is the matrix size.
Thus, the computation complexity of SMA is similar to that of
classic matrix approximation methods, such as regularized SVD~\cite{paterek2007improving}.

\section{Experiments}
\label{sec:experiment}

In this section, we first introduce the experimental setup, including
dataset description, parameter setting, evaluation metrics and compared methods.
Then, we analyze the performance of SMA for the rating prediction task
from the following aspects:
1) generalization performance of SMA;
2) sensitivity of SMA with different parameters, e.g.,
rank $r$ and the number of non-overlapping subsets $K$;
3) accuracy comparison between SMA and seven state-of-then-art matrix approximation
based recommendation algorithms, including four single MA methods and three ensemble methods;
and 4) SMA's accuracy in different data sparsity settings.
After that, we analyze the performance of SMA for the top-N recommendation task
from the following aspects:
1) generalization performance of SMA;
2) sensitivity of SMA with different parameters, e.g., surrogate loss functions and
$\gamma$ values;
and 3) accuracy comparison between SMA and four state-of-the-art top-N recommendation algorithms.

\subsection{Experimental Setup}
{\bf Data Description.} Four popular datasets are adopted to evaluate SMA:
1) MovieLens\footnote{https://grouplens.org/datasets/movielens/} 100K (943 users, 1682 movies, $10^5$ ratings);
2) MovieLens 1M ({\scriptsize${\sim}$}6k users, 4k items, ${10^6}$ ratings);
3) MovieLens 10M ({\scriptsize${\sim}$}70k users, 10k items, ${10^7}$ ratings) and
4) Netflix ({\scriptsize${\sim}$}480k users, 18k items, ${10^8}$ ratings).
These datasets are rating-based, so we follow other works~\cite{rendle09,Ning11} to
transform the data to binary data for the top-N recommendation task
in which the models predict whether a user will rate an item or not.
For each dataset, we randomly split it into training and test sets and keep the
ratio of training set to test set as 9:1. All experimental results are presented
by averaging the results over five different random train-test splits.

{\noindent\bf Parameter Setting.} For the rating prediction task, we use learning rate $v = 0.001$ for stochastic
gradient decent, $\mu_1 = 0.06$ for $L_2$-regularization coefficient,
$\epsilon = 0.0001$ for gradient descent convergence threshold,
and $T = 250$ for maximum number of iterations.
For the top-N recommendation task, we use learning rate $v = 0.001$ for
stochastic gradient decent, $\mu_1 = 0.001$ for $L_2$-regularization coefficient,
$\epsilon = 0.0001$ for gradient descent convergence threshold,
and $T = 2000$ for maximum number of iterations. Optimal parameters
of the compared methods are chosen from their original papers.
Note that adaptive learning rate methods, e.g., AdaError~\cite{adaerror}, can be
adopted to further improve the accuracy of the proposed method.

{\noindent\bf Evaluation Metrics.} The accuracy of rating prediction is evaluated using root mean
square error (RMSE) as follows:
$\sqrt{\frac{1}{|\Omega|}\sum_{(i,j)\in\Omega}{(R_{i,j}-\hat{R}_{i,j})^2}}$.
The accuracy of top-N recommendation is evaluated using Precision and
normalized discounted cumulative gain (NDCG) as follows:
1) Precision$@$N for a targeted user $u$ can be computed as follows:
$Precision@N = {|I_r\cap I_u|}/|I_r|$
where $I_r$ is the list of top N recommendations and $I_u$ is the list of items
that $u$ is interested in and
2) NDCG$@$N for a targeted user $u$ can be computed as follows:
$NDCG@N=DCG@N/IDCG@n$, where $DCG@N = \sum_{k=1}^n{(2^{rel_i}-1)/log_2(i+1)}$
and $IDCG@n$ is the value of $DCG@N$ with perfect ranking
($rel_i=1$ if $u$ is interested to the $i$-th recommendation and $rel_i=0$ otherwise).

{\noindent\bf Compared Methods.} For rating prediction task, we compare the performance of SMA with
four single matrix approximation methods and three ensemble matrix approximation methods as follows:
\begin{itemize}
    \item Regularized SVD~\cite{paterek2007improving} is a popular matrix
    factorization method, in which user/item features are estimated
    by minimizing the sum-squared error using $L_2$ regularization.
    \item BPMF~\cite{salakhutdinov2008bayesian} is a Bayesian extension of PMF
    with model parameters and hyperparameters estimated using Markov chain Monte Carlo method.
    \item APG~\cite{toh2010accelerated} computes the approximation by
    solving a nuclear norm regularized linear least squares problem to speed up convergence.
    \item GSMF~\cite{yuan2014recommendation}
    can transfer information among multiple types of user behaviors
    by modeling the shared and private latent factors with group sparsity regularization.
    \item DFC~\cite{mackey2011divide} is an ensemble method,
    which divides a large-scale matrix factorization task into
    smaller subproblems, solves each other in parallel, and finally combines the subproblem solutions to improve accuracy.
    \item LLORMA~\cite{lee2013local} is an ensemble method, which assumes that the original matrix is
    described by multiple low-rank submatrices constructed by non-parametric kernel smoothing techniques.
    \item WEMAREC~\cite{Chen15} is an ensemble method, which constructs biased model by
    weighting strategy to address the insufficient data issue in each submatrix and integrates
    different biased models to achieve higher accuracy.
\end{itemize}

For top-N recommendation, we compare SMA with the following four state-of-the-art
collaborative filtering methods:
\begin{itemize}
\item WRMF~\cite{Hu08} assigns point-wise confidences to different ratings
so that positive examples will have much larger confidence than negative examples.
\item BPR~\cite{rendle09} minimizes a pair-wise loss function to optimize ranking
measures, e.g., NDCG. They proposed different versions of BPR methods, e.g., BPR-kNN
and BPR-MF. We compare SMA with the BPR-MF, which is also MA-based recommendation method;
\item AOBPR~\cite{Rendle14} improves the original BPR method by a non-uniform
item sampler and meanwhile oversamples informative pairs to speed up convergence.
\item SLIM~\cite{Ning11} generates top-N recommendations by aggregating user rating
profiles using sparse representation.
In their method, the weight matrix of user ratings is obtained by solving an $L_1$
and $L_2$ regularized optimization problem.
\end{itemize}

\subsection{Rating Prediction Evaluation}

\subsubsection{Generalization Performance}

\begin{figure*}[tbh!]
    \centering
      \includegraphics[width=0.5\textwidth]{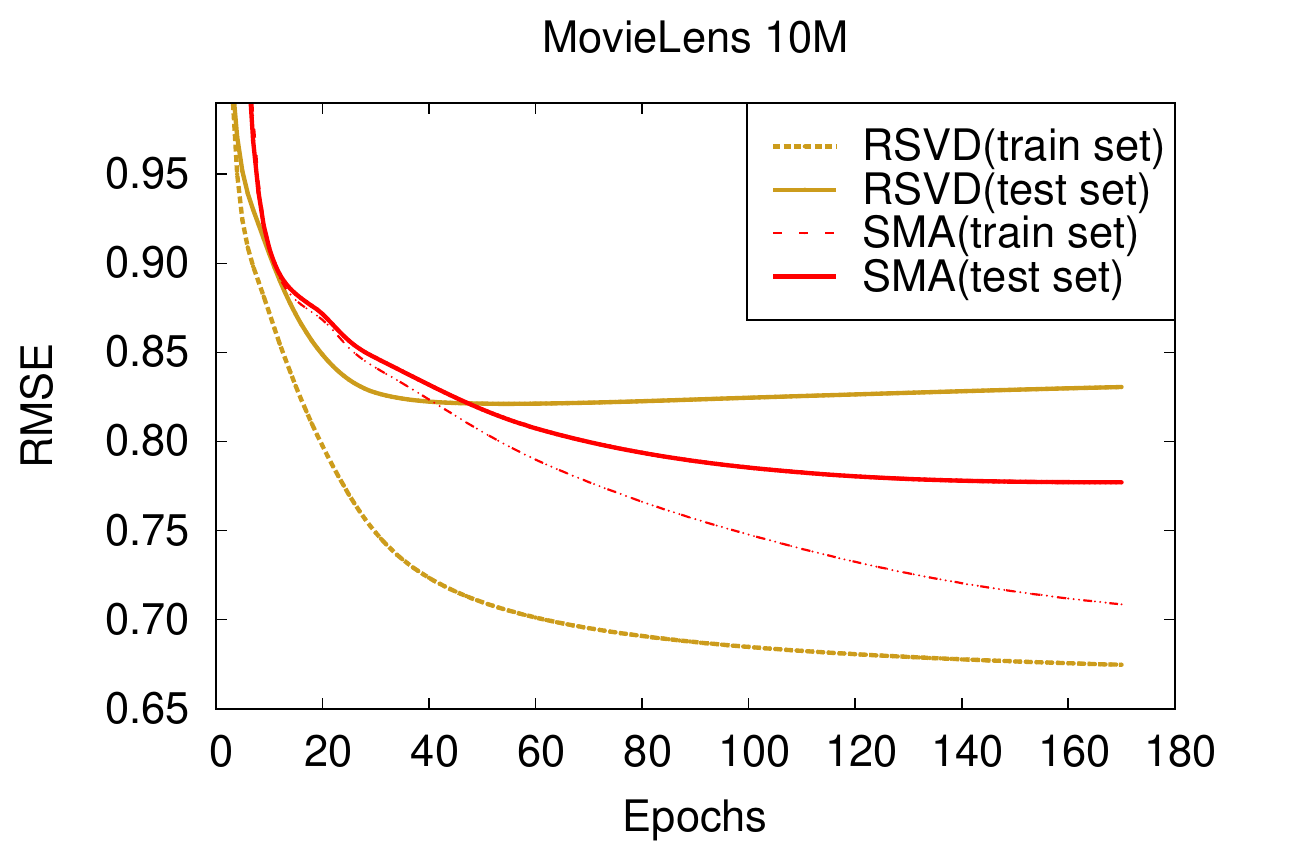}
      \caption{Training and test errors vs. epochs of RSVD and SMA on the MovieLens 10M dataset.
       \label{fig:rmseOfEpoch}}
\end{figure*}

Figure~\ref{fig:rmseOfEpoch} compares training/test errors of SMA
and RSVD with different epochs on the MovieLens 10M dataset (rank $r = 20$ and subset number $K=3$).
As we can see, the differences between training and test error
of SMA are much smaller than that of RSVD. Moreover, the training error and test error
are very close when epoch is less than 100. This result demonstrates
that SMA can indeed find models that have good generalization performance and
yield small generalization error during the training process.

\subsubsection{Sensitivity Analysis}
\label{sec:exp}

\begin{figure*}[tbh!]
      \centering$
      \begin{array}{cc}
      \includegraphics[width=0.49\textwidth]{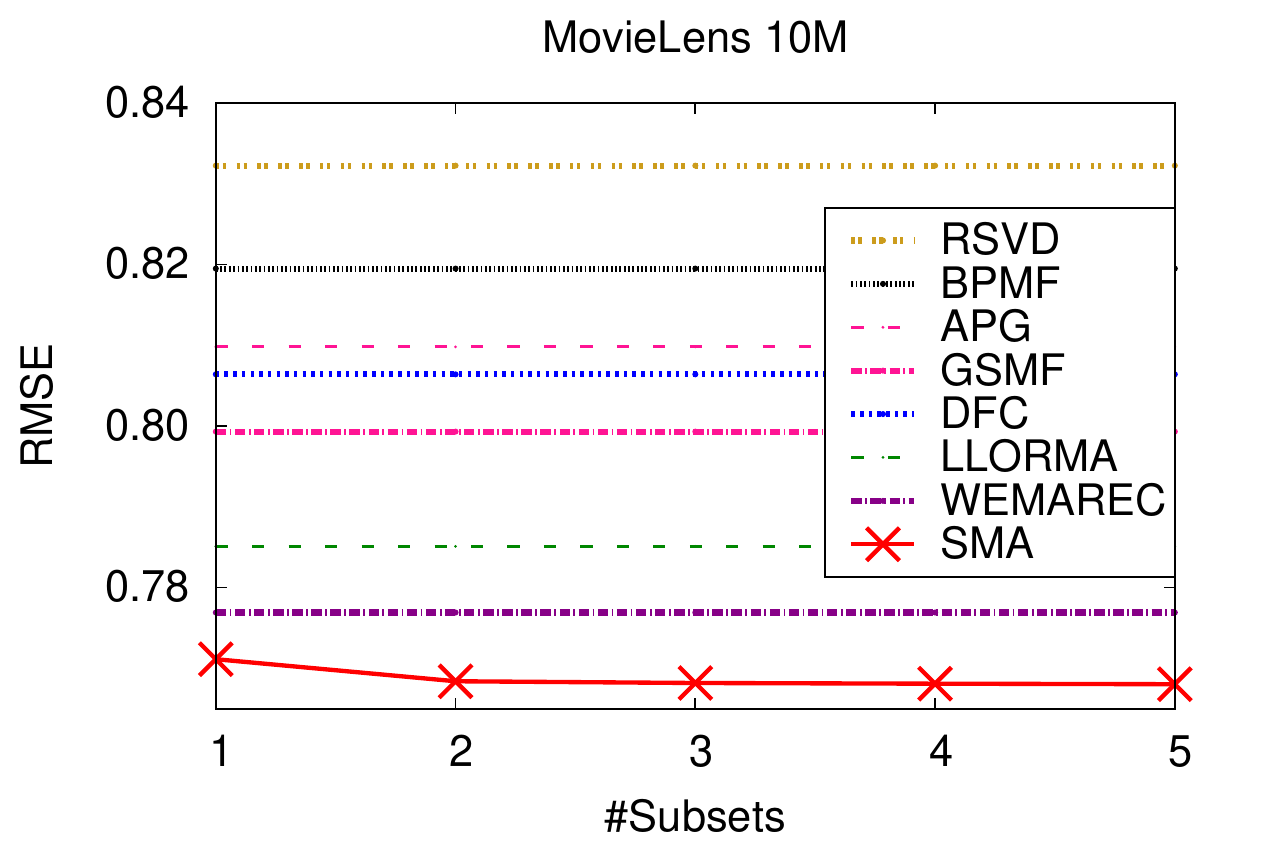}
      \includegraphics[width=0.49\textwidth]{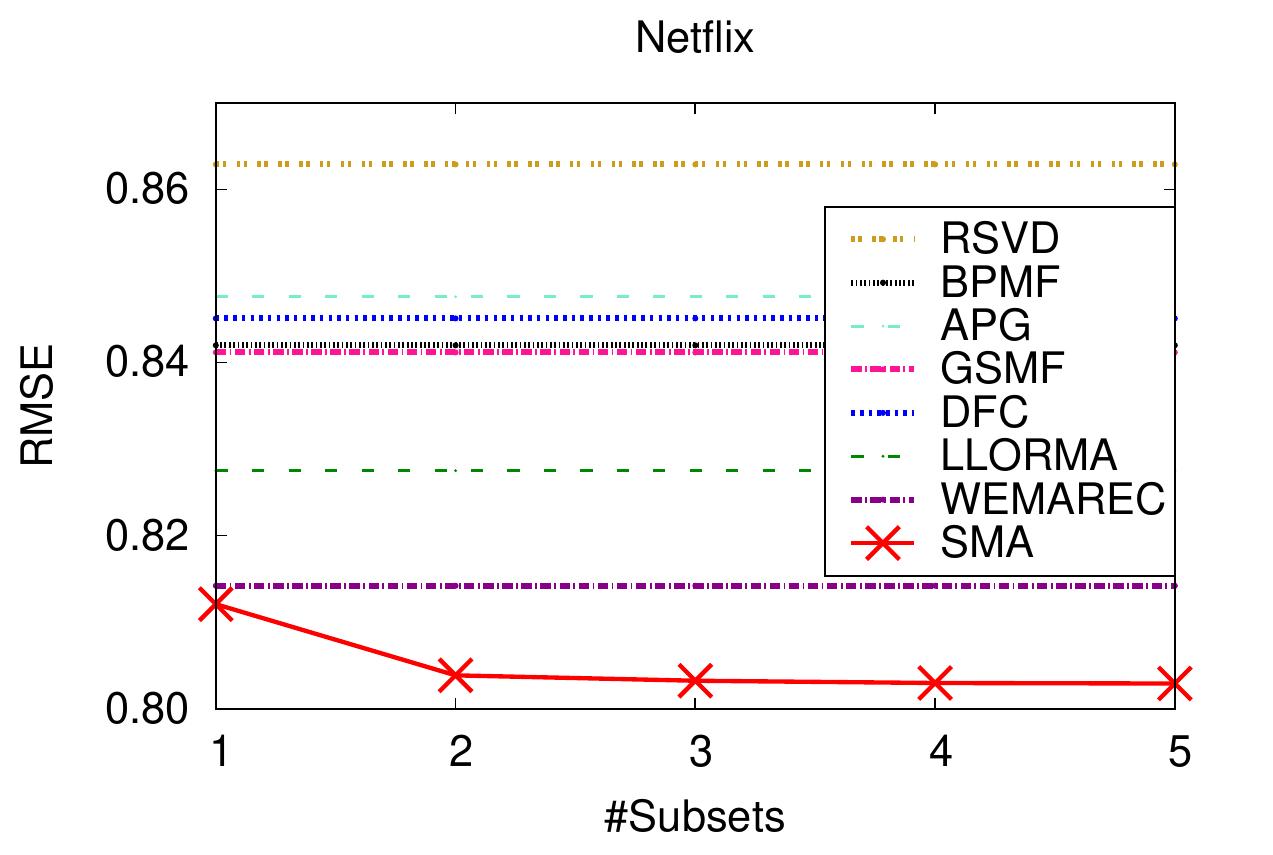}
      \end{array}$
      \caption{Effect of subset number $K$ on the MovieLens 10M dataset (left) and Netflix dataset (right).
    SMA models are indicated by solid lines and other compared methods are indicated by dotted lines.}
    \label{fig:effectOfK}
\end{figure*}

\begin{figure*}[tbh!]
      \centering$
      \begin{array}{cc}
      \includegraphics[width=0.49\textwidth]{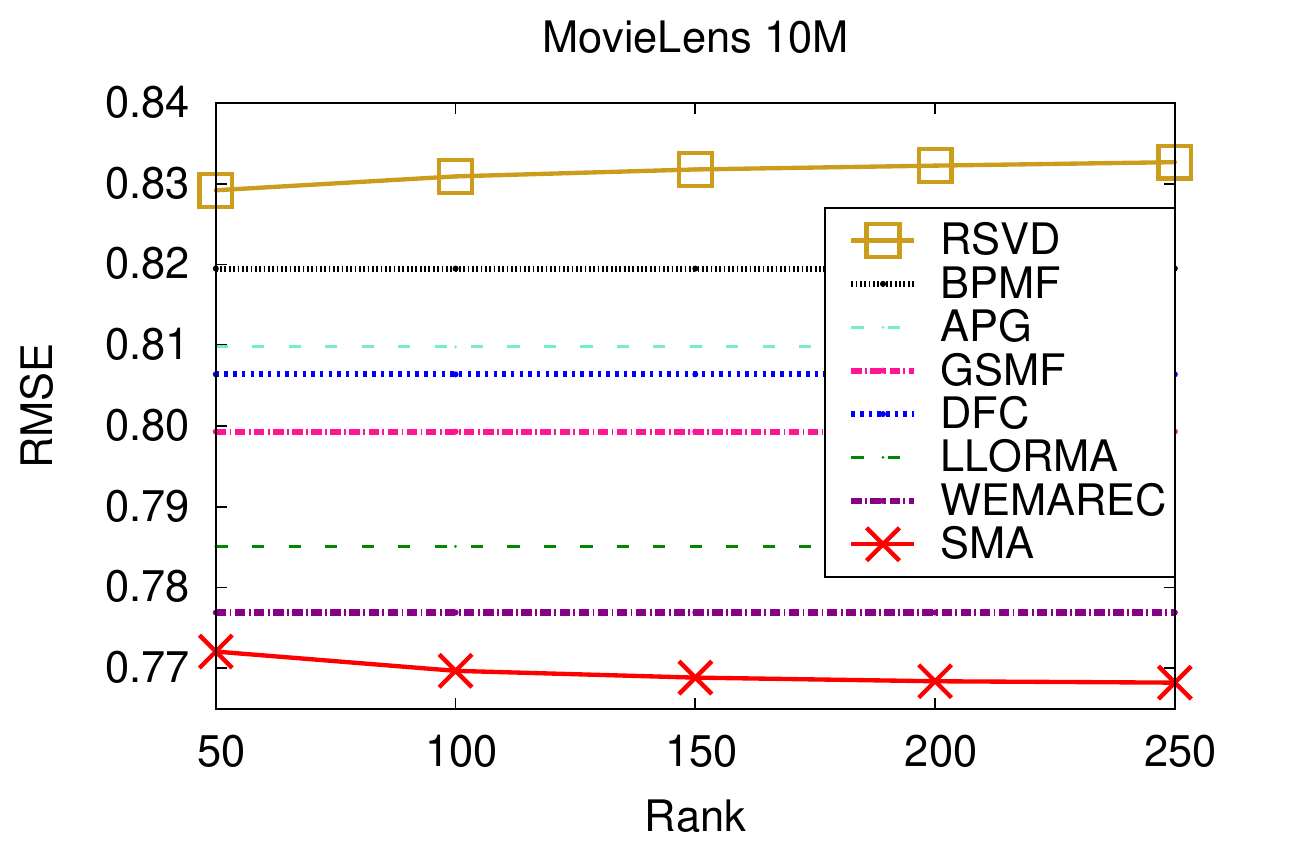}
      \includegraphics[width=0.49\textwidth]{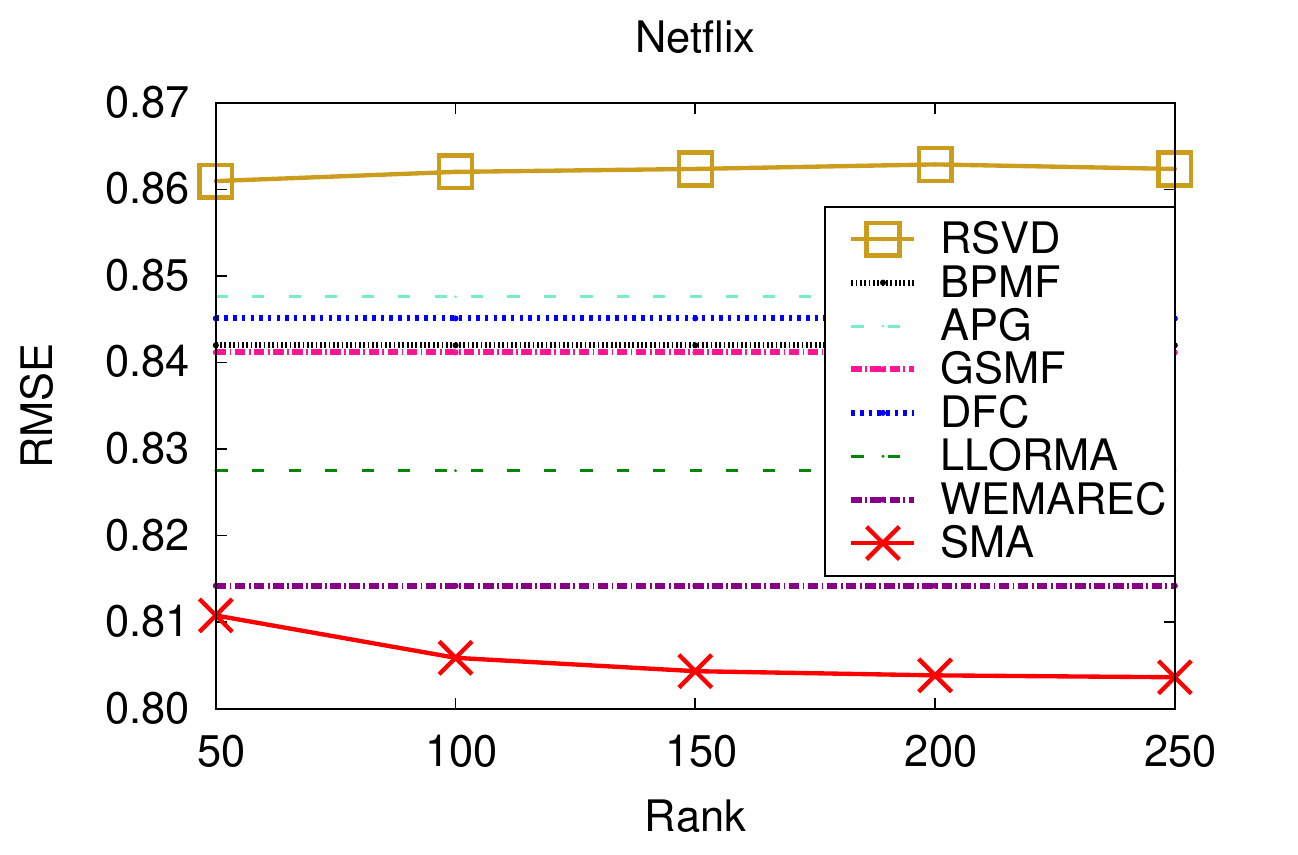}
      \end{array}$
      \caption{Effect of rank $r$ on the MovieLens 10M dataset (left) and Netflix dataset (right).
      SMA and RSVD models are indicated by solid lines and other compared methods
      are indicated by dotted lines.}
      \label{fig:effectOfRank}
\end{figure*}

Figure~\ref{fig:effectOfK}  investigates how SMA performs by
varying number of non-overlapping subsets $K$ (rank $r=200$)
and the optimal RMSEs of all compared methods on both Movielens 10M (left) and Netfilx (right) datasets.
As we can see, SMA outperforms all these state-of-the-art methods with $K$ varying from 1 to 5.
It should be noted that, when $K=0$, SMA is degraded to RSVD.
Thus, the fact that SMA can produce better recommendations than RSVD
confirms Theorem~\ref{thm:stable_1}:
with additional terms $\sum_{s=1}^K{\lambda_s\mathcal{D}_{\Omega_s}(\hat{R})}$,
we can improve the stability of MA models.
In addition, we can see that the RMSEs on both datasets decrease
as $K$ increases. 
This further confirms Theorem~\ref{thm:stable_4}:
probing easily predictable entries to form harder predictable
entry sets can better increase the model performance.

Figure~\ref{fig:effectOfRank} analyzes the effect of rank $r$
on MovieLens 10M (left) and Netflix (right) datasets by fixing $K=3$.
It can be seen that for any rank $r$ from 50 to 250, SMA always
outperforms the other seven compared methods in recommendation accuracy.
And higher ranks for SMA will lead to better accuracy
when the rank $r$ increases from $50$ to $250$ on both two datasets.
It is interesting to see that the recommendation accuracies of RSVD
decrease slightly when $r>50$ due to over-fitting and SMA can consistently
increase recommendation accuracy even when $r>200$.
This indicates that SMA is less prone to over-fitting
than RSVD, i.e., SMA is more stable than RSVD.

\subsubsection{Accuracy Comparisons}
\begin{table}[tb!]
\caption{RMSEs of SMA and the seven compared methods on MovieLens (10M) and Netflix datasets.
}
\centering
\begin{tabular}{| c | c | c |}
\hline
                        &  MovieLens (10M)         &  Netflix        \\
\hline
  RSVD                  &   0.8256 $\pm$ 0.0006    &    0.8534 $\pm$ 0.0001        \\
\hline
  BPMF                  &   0.8197 $\pm$ 0.0004    &    0.8421 $\pm$ 0.0002       \\
\hline
  APG                   &   0.8101 $\pm$ 0.0003    &    0.8476 $\pm$ 0.0003        \\
\hline
  GSMF                  &   0.8012 $\pm$ 0.0011    &    0.8420 $\pm$ 0.0006         \\
\hline
  DFC                   &   0.8067 $\pm$ 0.0002    &    0.8453 $\pm$ 0.0003        \\
\hline
  LLORMA                &   0.7855 $\pm$ 0.0002    &    0.8275 $\pm$ 0.0004        \\
\hline
  WEMAREC               &   0.7775  $\pm$ 0.0007   &    0.8143  $\pm$ 0.0001        \\
\hline
  \textbf{SMA}   &   \textbf{0.7682 $\pm$ 0.0003}   &    \textbf{0.8036 $\pm$ 0.0004}       \\
\hline
\end{tabular}
\label{tbl:accuracyCmp}
\end{table}


Table~\ref{tbl:accuracyCmp} presents the performance of
SMA with rank $r = 200$ and subset number $K = 3$.
The compared methods are as follows:
RSVD ($r=50$)~\cite{paterek2007improving}, BPMF ($r=300$)~\cite{salakhutdinov2008bayesian},
APG ($r=100$)~\cite{toh2010accelerated}, GSMF ($r=20$)~\cite{yuan2014recommendation},
DFC ($r=30$)~\cite{mackey2011divide},
LLORMA ($r=20$)~\cite{lee2013local} and WEMAREC ($r=20$)~\cite{Chen15}
on MovieLens 10M and Netflix datasets.
Notably, DFC, LLORMA and WEMAREC are ensemble methods,
which have been shown to be more accurate than single methods due to better generalization performance.
However, as shown in Table~\ref{tbl:accuracyCmp}, our SMA method
significantly outperforms all seven compared methods on both two datasets.
This confirms that SMA can indeed achieve better generalization performance
than both state-of-the-art single methods and ensemble methods. The main reason is that
SMA can minimize objective functions that lead to solutions with good generalization performance,
but other methods cannot guarantee low gap between training error
and test error.

\subsubsection{Performance under Data Sparsity}

\begin{figure*}[tbh!]
      \centering
      \includegraphics[width=0.5\textwidth]{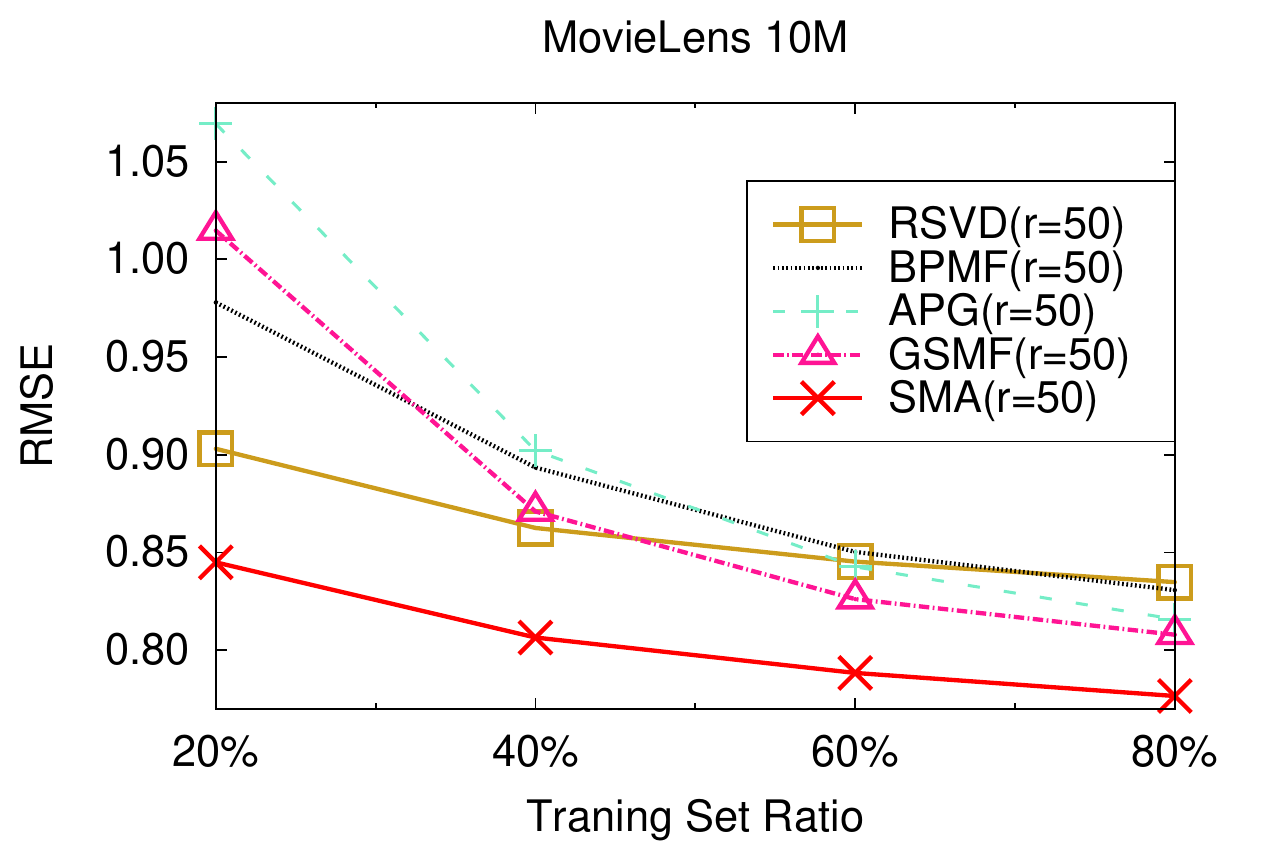}
      \caption{RMSEs of SMA and four single methods with varying training set size
      on the MovieLens 10M dataset (rank $r=50$).}
      \label{fig:trainRatio}
\end{figure*}

Figure~\ref{fig:trainRatio} presents the RMSEs of SMA vs. the training set size
as compared with four single LRMA methods (RSVD, BPMF, APG and GSMF).
The rank $r$ of all five methods are fixed to $50$. Note that, the rating
density becomes more sparse when the training set ratio decreases.
The results show that all methods can improve accuracy when the training set size
increases, but the proposed SMA method always outperforms the other methods.
This demonstrates that SMA can still provide stable matrix approximation even on
very sparse dataset.

\subsection{Top-N Recommendation Evaluation}

\subsubsection{Generalization Performance}

\begin{figure*}[tbh!]
      \centering$
      \begin{array}{ccc}
      \includegraphics[width=0.33\textwidth]{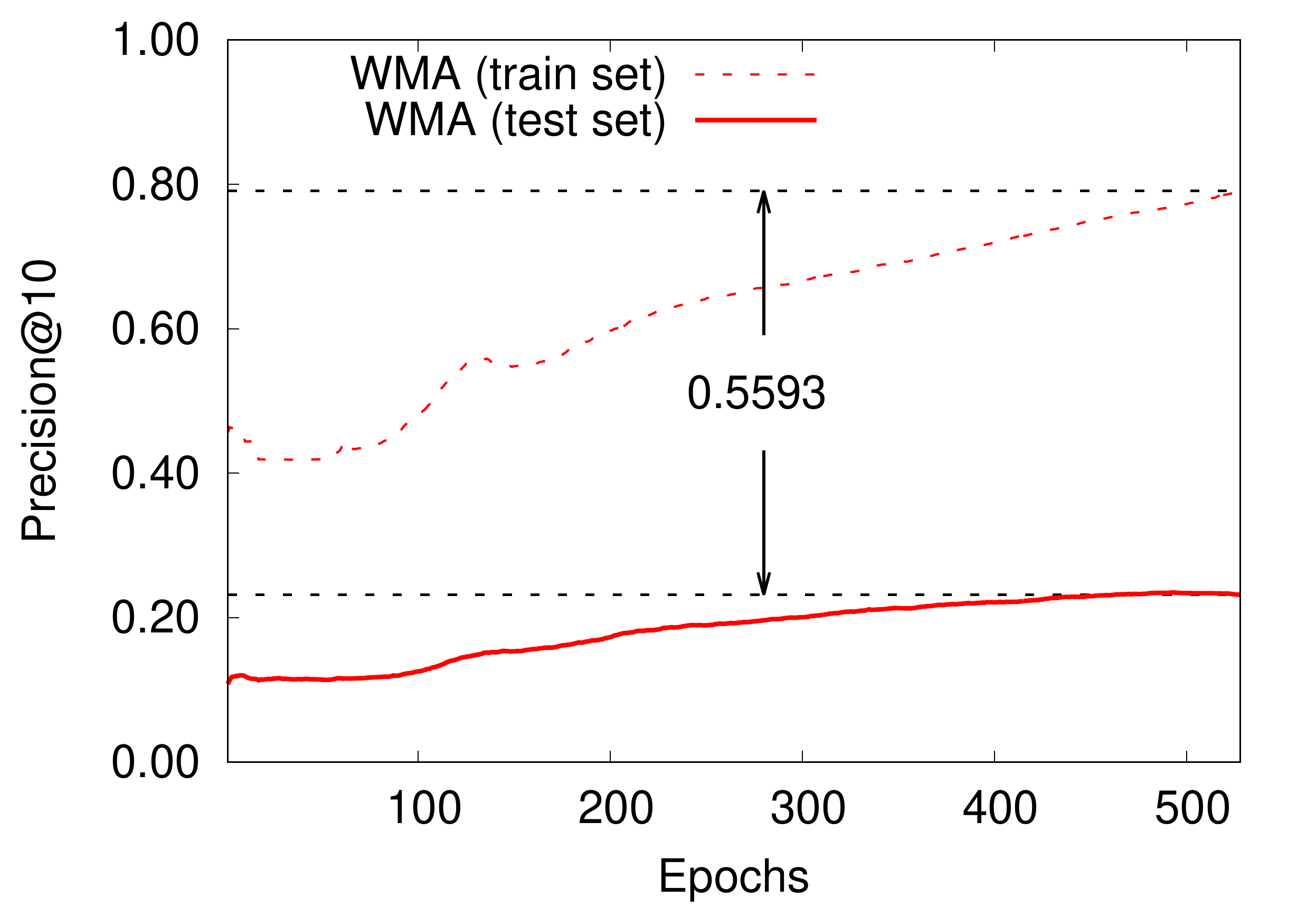}
      \includegraphics[width=0.33\textwidth]{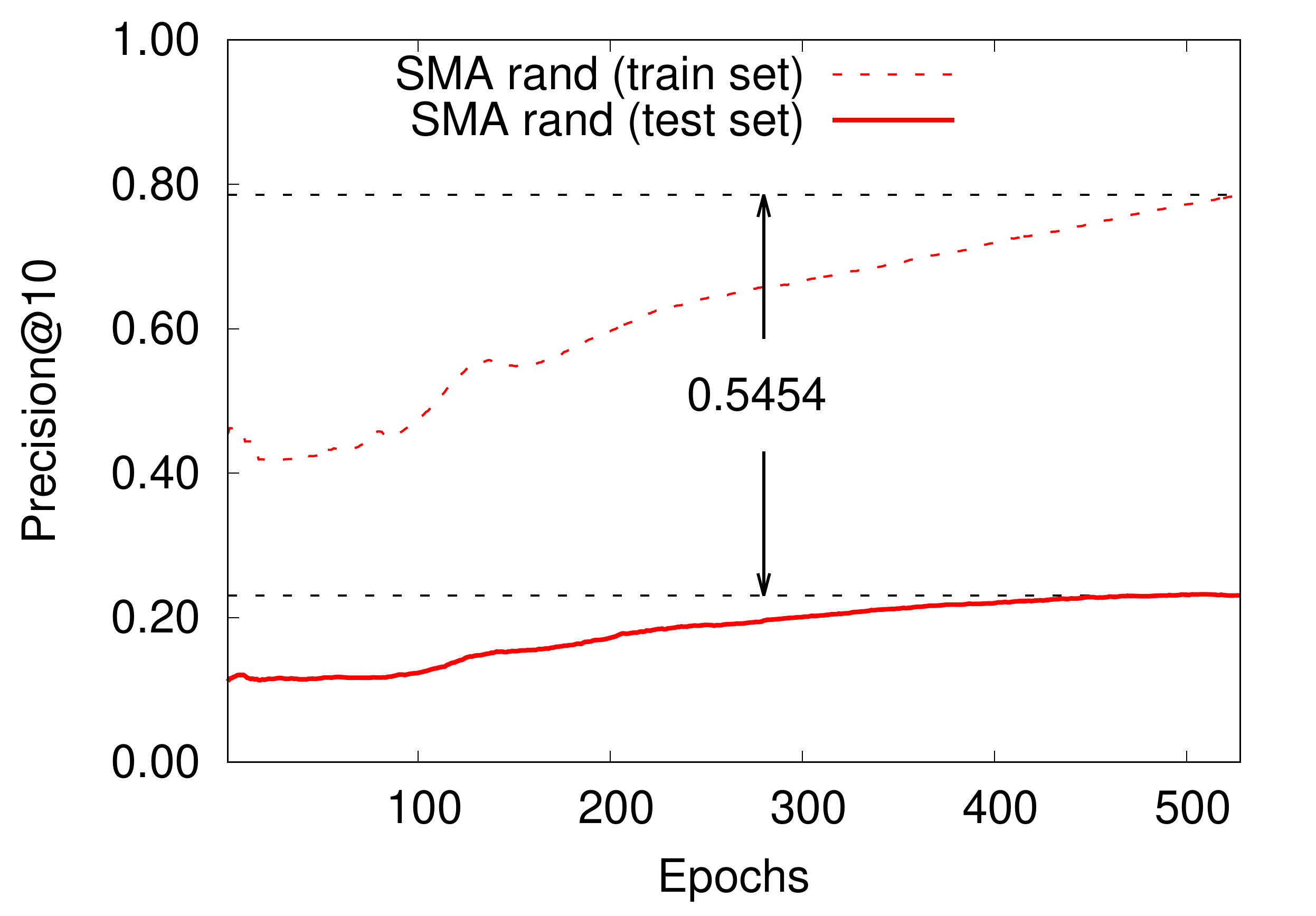}
      \includegraphics[width=0.33\textwidth]{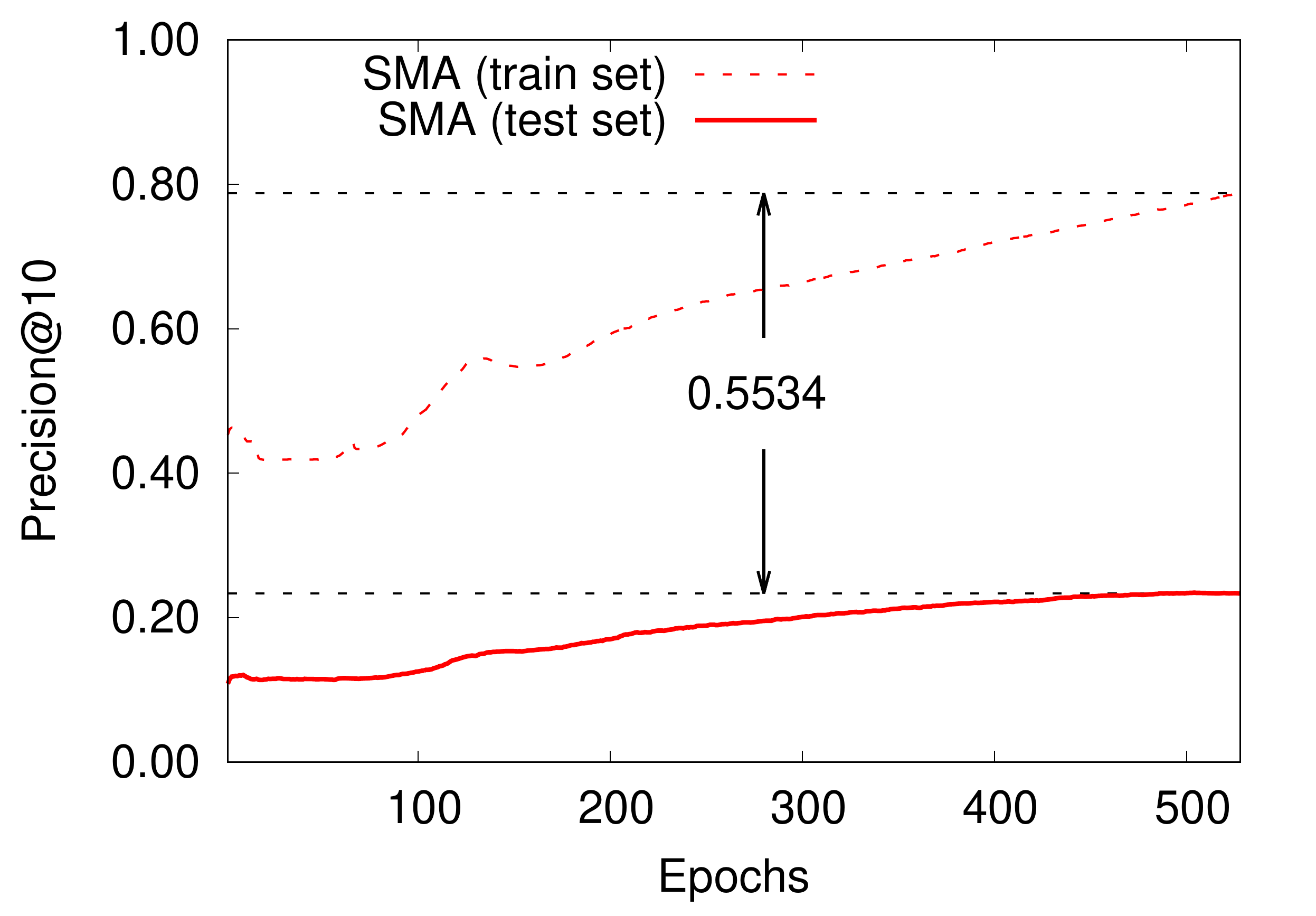}
      \end{array}$
      \caption{Generalization performance analysis of SMA with weighted Exp loss on MovieLens 100K dataset.
      The leftmost figure shows a weighted MA (WMA) methods without selecting $\Omega'$,
      which is equivalent to SMA with $\Omega'=\emptyset$.
      The middle figure shows SMA with random selected $\Omega'$.
      The rightmost figure shows SMA with hard predictable $\Omega'$, i.e.,
      examples with predicted ratings in [-0.3,0.3] are selected as $\Omega'$.}
    \label{fig:gen_topn}
\end{figure*}

Figure~\ref{fig:gen_topn} compares the generalization performance of three different
matrix approximation methods in the top-N setting:
1) weighted matrix approximation (WMA), which is the same as SMA except that WMA has no
additional terms in the loss function, i.e., WMA learns MA models by optimizing Equation~\ref{eqn:weighted_loss};
2) stable matrix approximation with random $\Omega'$ (SMA rand), in which $\Omega'$ is randomly selected
rather than selecting the examples near the decision boundary;
and 3) stable matrix approximation, in which $\Omega'$ is selected as the examples
near the decision boundary.

As shown in Figure~\ref{fig:gen_topn}, both SMA rand and SMA can achieve lower gap between
training precision@10 and test precision@10 than WMA, which confirms with Theorem~\ref{thm:stable_5}:
adding the term $\mathcal{D}_{\Omega'}$ in the loss function can improve the stability of MA models,
i.e., improve the generalization performance. For SMA rand, we can see that, although its
generalization performance improves, its optimization performance (training accuracy) is affected
due to the random selection of  $\Omega'$. This demonstrates that random selection of $\Omega'$
can improve generalization performance but degrade optimization performance simultaneously.
Therefore, the test accuracy of SMA rand has no improvement compared with WMA.
For SMA with hard predictable $\Omega'$, the optimization performance
(training accuracy) is not affected but improved compared with WMA.
This is because examples near the decision boundary are given larger gradients updates in SMA,
which can help improve the optimization performance because examples near the decision
boundary have higher optimization error.
Meanwhile, compared with WMA, the generalization performance of SMA is also improved, which is
due to the stability improvement from the selection of $\Omega'$. Since both the optimization performance and
generalization performance are improved in SMA for top-N recommendation, SMA achieves significantly
higher test accuracy than WMA.

This experiment demonstrates that selecting random subsets can improve generalization performance but
cannot improve test accuracy due to the degradation of training accuracy. On the contrary, adding
properly selected subsets, e.g., examples near the decision boundary, can improve generalization
performance and does not hurt training accuracy, and can thus improve test accuracy.

\subsubsection{Sensitivity of Surrogate Losses}

\begin{figure*}[tbh!]
      \centering$
      \begin{array}{ccc}
      \includegraphics[width=0.33\textwidth]{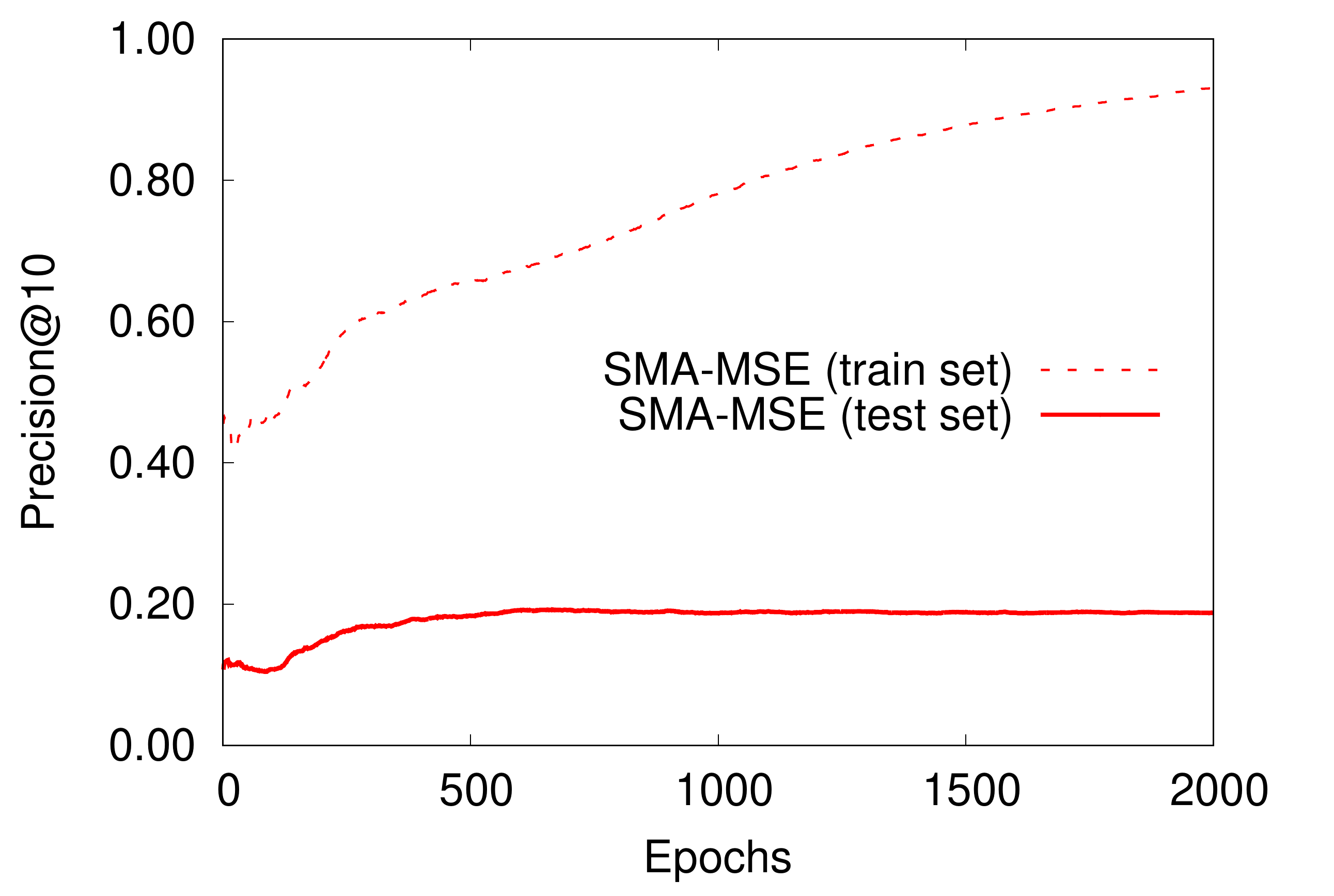}
      \includegraphics[width=0.33\textwidth]{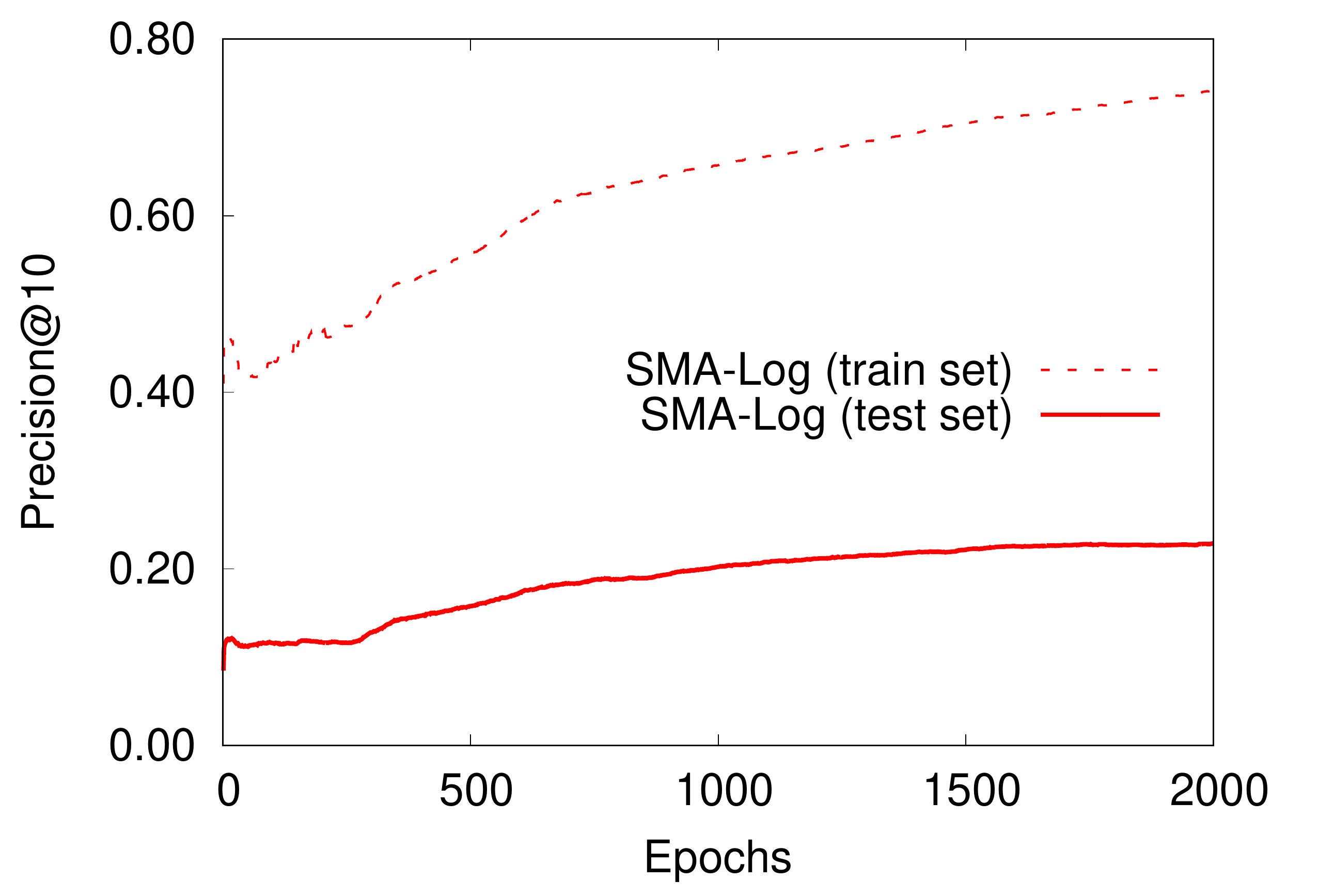}
      \includegraphics[width=0.33\textwidth]{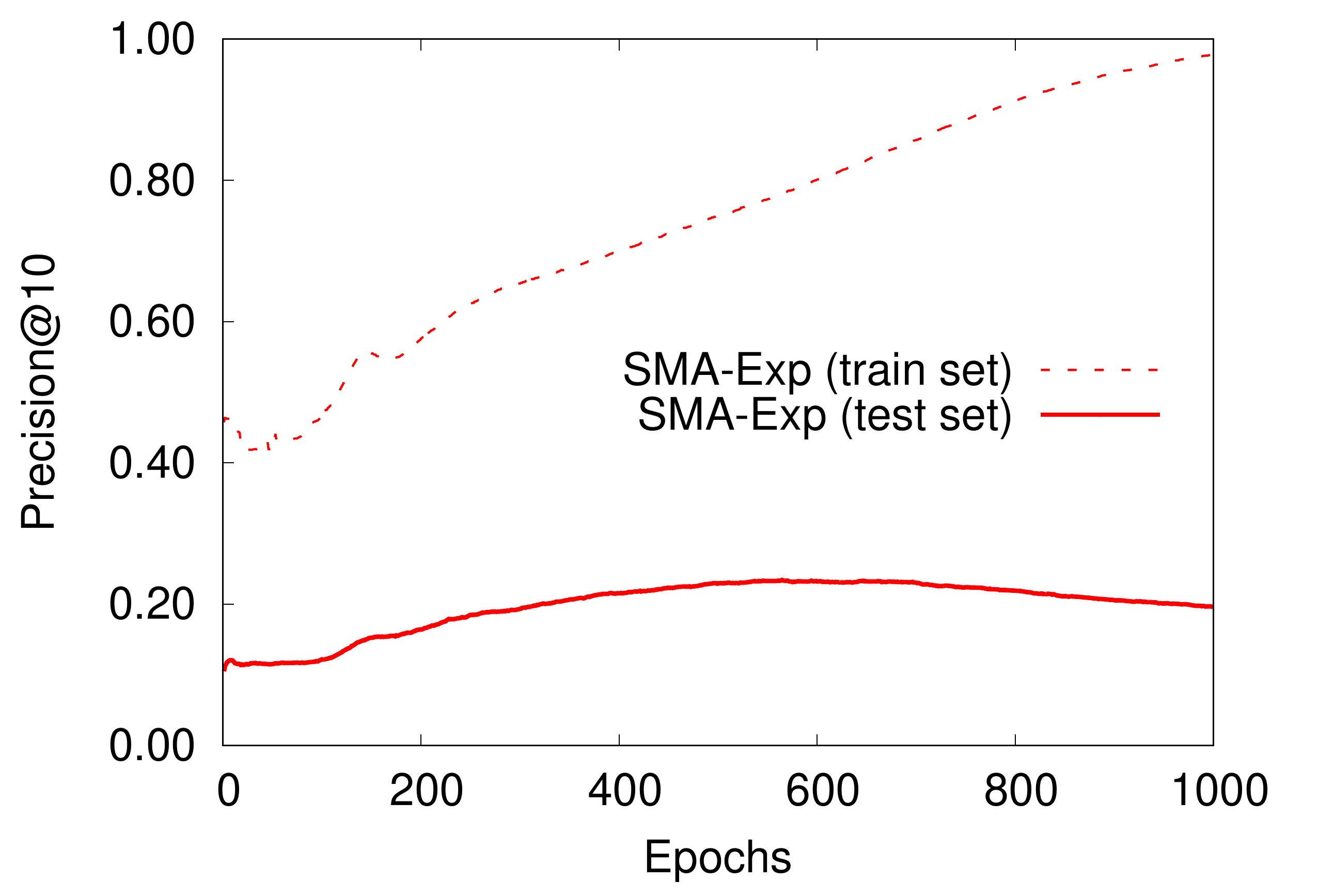}
      \end{array}$
      \caption{Comparison of SMA with different surrogate losses on the MovieLens 100K dataset.
      Here, we choose $\gamma = 0.3$ and rank $r=100$.}
    \label{fig:surrogate}
\end{figure*}
Figure~\ref{fig:surrogate} compares the performance of SMA with three different surrogate loss functions:
1) mean square error (MSE); 2) log loss (Log) and 3) exponential loss (Exp). As shown in Figure~\ref{fig:surrogate},
SMA with mean square loss achieves much worse test accuracy compared with SMA with log loss and
exponential loss, which is because mean square loss is a rating-based loss rather than a ranking-based
loss. SMA with log loss and exponential loss achieve very similar test accuracy. However, the convergence
speed of SMA with exponential loss is much faster (around 500 epochs) than that of SMA with log loss (around 2000 epochs).
Therefore, SMA with exponential loss is more desirable due to decent accuracy and faster convergence speed.

\subsubsection{Accuracy vs. $\gamma$}

\begin{figure*}[tbh!]
      \centering$
      \begin{array}{cc}
      \includegraphics[width=0.475\textwidth]{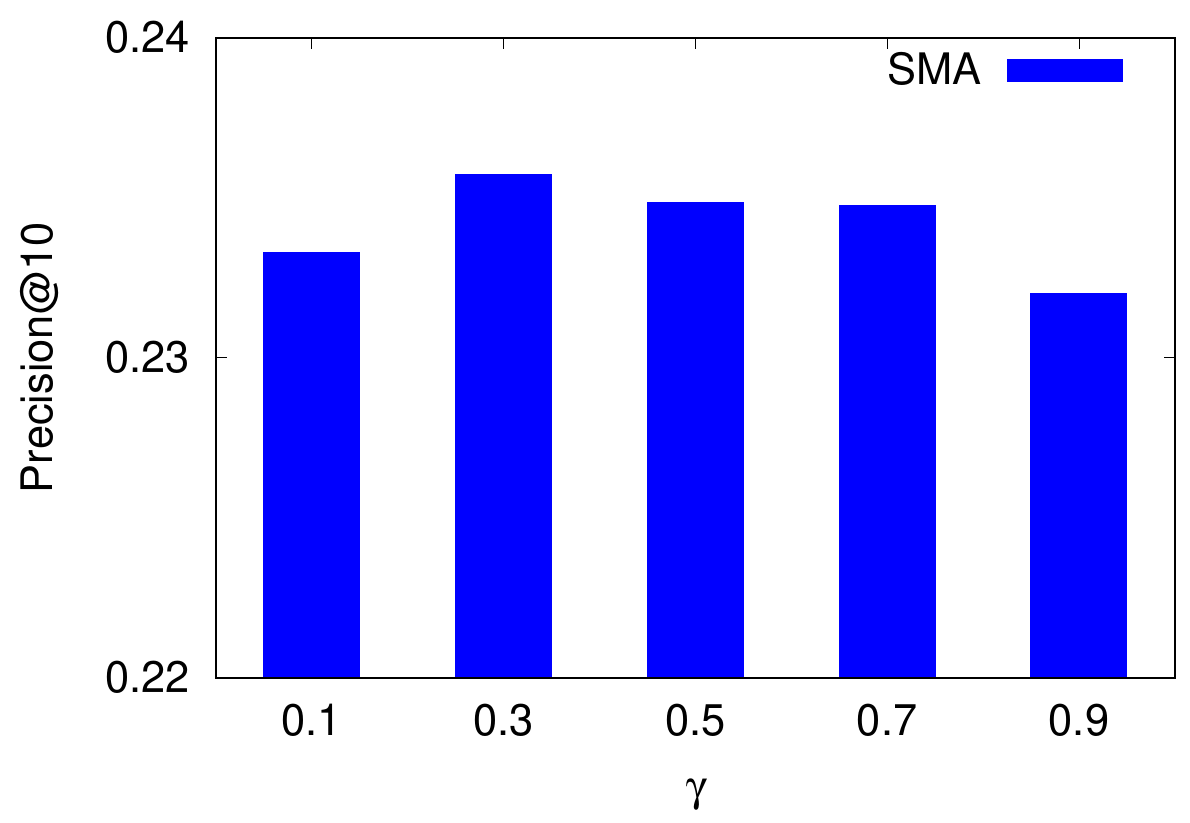}\hfill
      \includegraphics[width=0.49\textwidth]{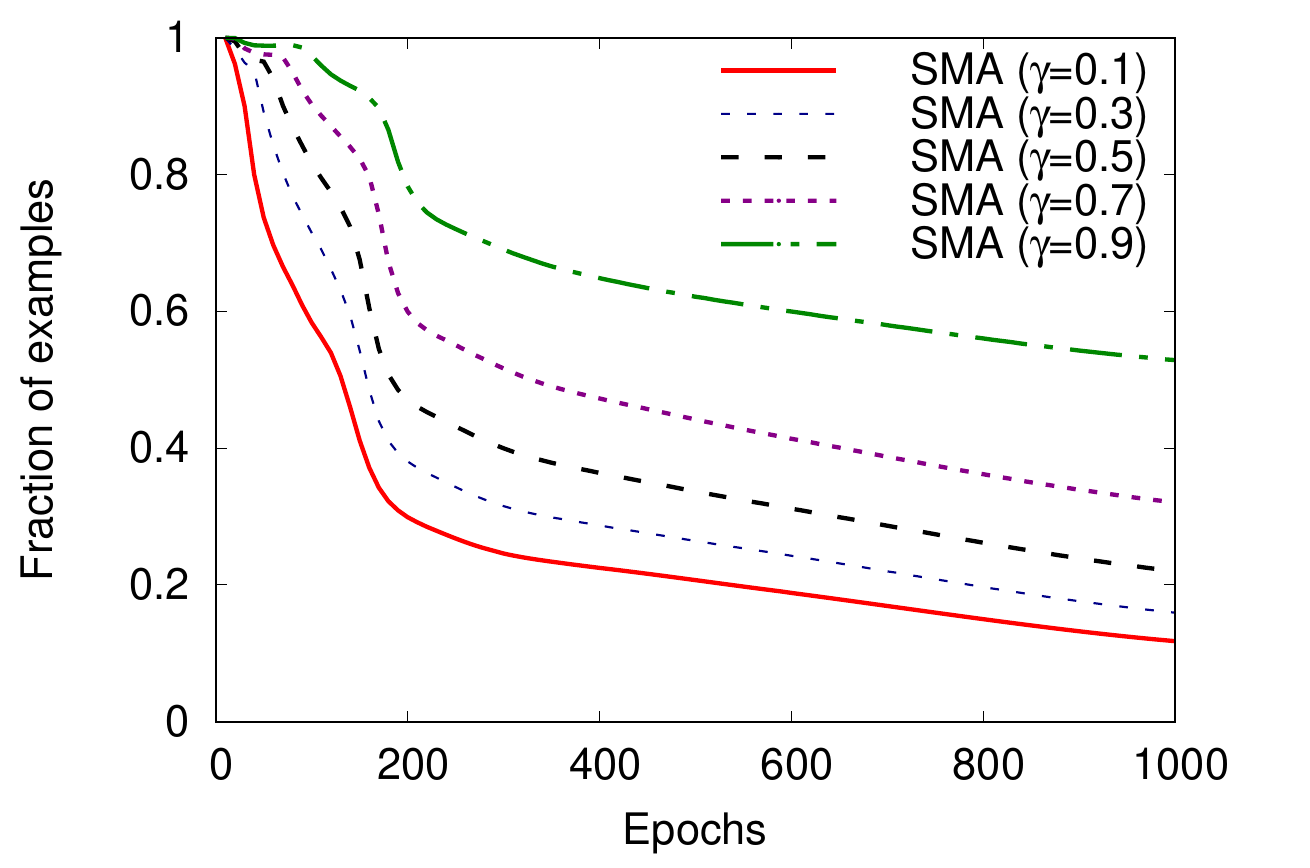}
      \end{array}$
      \caption{Accuracy vs. $\gamma$ values of SMA on MovieLens 100K dataset.
      The figure on the left shows how the Precision@10 varies as $\gamma$ changes from 0.1 to 0.9,
      and the figure on the right shows the ratio of selected examples in $\Omega'$ with different
      $\gamma$ values.}
    \label{fig:gamma}
\end{figure*}

Figure~\ref{fig:gamma} compares the performance of SMA with different $\gamma$ values, in which
training examples with predicted ratings ranging in $[-\gamma, \gamma]$ are selected as $\Omega'$ for each epoch.
The figure on the left shows how the recommendation accuracy varies as $\gamma$ increases from 0.1 to 0.9,
and the figure on the right shows the ratio of selected examples in $\Omega'$ with different
$\gamma$ values.
As shown in Figure~\ref{fig:gamma}, SMA with $\gamma=0.3$ achieves the best test accuracy.
For smaller $\gamma$, e.g., 0.1, or larger $\gamma$, e.g., 0.9, SMA achieves worse accuracy
because too small or too large fraction of examples are chosen in $\Omega'$.
This further confirms that the accuracy of SMA will vary with different $\Omega'$ and
properly selected $\Omega'$ can help achieve better accuracy. 

\subsubsection{Accuracy Comparison}

\begin{table*}[tb!]
\caption{Precision comparison between SMA
and one rating-based MA method (RSVD) and four state-of-the-art top-N recommendation methods (BPR, WRMF, AOBRP, SLIM)
on the Movielens 1M and Movielens 100K datasets.
Bold face means that SMA statistically significantly outperforms the other methods with 95\% confidence level.
}
\centering
\normalsize
\begin{tabular}[t]{|c|c|c|c|c|c|c|c|c|c|}
\hline
\multicolumn{2}{|c|} {Metric} & \multicolumn{4}{c|} {Precision$@$N} \\
\hline
\multicolumn{2}{|c|}{Data $|$ Method} & N=1 & N=5 & N=10 & N=20 \\
\hline
\multirow{6}{*}{\rotatebox{90}{ML-1M}}    & RSVD    & 0.1659 $\pm$ 0.0017 & 0.1263 $\pm$ 0.0005 & 0.1037 $\pm$ 0.0009
                                                    & 0.0766 $\pm$ 0.0020 \\
                                            & BPR   & 0.3062 $\pm$ 0.0030 & 0.2277 $\pm$ 0.0074 & 0.1896 $\pm$ 0.0048
                                                    & 0.1516 $\pm$ 0.0007 \\
                                            & WRMF  & 0.2761 $\pm$ 0.0074 & 0.2155 $\pm$ 0.0009 & 0.1816 $\pm$ 0.0007
                                                    & 0.1459 $\pm$ 0.0004 \\
                                            & AOBPR & 0.3098 $\pm$ 0.0076 & 0.2315 $\pm$ 0.0002 & 0.1926 $\pm$ 0.0022
                                                    & {0.1540 $\pm$ 0.0016} \\
                                            & SLIM  & 0.3053 $\pm$ 0.0097 & 0.2208 $\pm$ 0.0039 & 0.1836 $\pm$ 0.0006
                                                    & 0.1419 $\pm$ 0.0029 \\
                                           & \textbf{SMA} & \textbf{0.5133 $\pm$ 0.0047} & \textbf{0.3937 $\pm$ 0.0064}
                          & \textbf{0.3258 $\pm$ 0.0061} & \textbf{0.2548 $\pm$ 0.0051} \\
\hline
\multirow{6}{*}{\rotatebox{90}{ML-100K}}    & RSVD   & 0.3155 $\pm$ 0.0038 & 0.2179 $\pm$ 0.0007 & 0.1403 $\pm$ 0.0035
                                                     & 0.1300 $\pm$ 0.0057 \\
                                            & BPR    & 0.3439 	$\pm$	0.0168 & 0.2533 	$\pm$	0.0082
                                                     & 0.2061 	$\pm$	0.0040 & 0.1581 	$\pm$	0.0028 \\
                                            & WRMF   & 0.3851	$\pm$	0.0116 & 0.2752 	$\pm$	0.0053
                                                     & 0.2202 	$\pm$	0.0056 & 0.1679 	$\pm$	0.0035 \\
                                            & AOBPR  & 0.3395 	$\pm$	0.0099 & 0.2591 	$\pm$	0.0057
                                                     & 0.2119 	$\pm$	0.0031 & 0.1632 	$\pm$	0.0025\\
                                            & SLIM   & 0.3951 	$\pm$	0.0056 & 0.2625 	$\pm$	0.0090
                                                     & 0.2055 	$\pm$	0.0031
                                                     & 0.1539 	$\pm$	0.0015 \\
                                           & \textbf{SMA} & \textbf{0.4179 	$\pm$	0.0072}
                                                            & \textbf{0.2931 	$\pm$	0.0023}
                                                            & \textbf{0.2347 	$\pm$	0.0033}
                                                            & \textbf{0.1772 	$\pm$	0.0022} \\
\hline
\end{tabular}
\label{tab:precision}
\end{table*}

\begin{table*}[tb!]
\caption{NDCG comparison between SMA
and one rating-based MA method (RSVD) and four state-of-the-art top-N recommendation methods (BPR, WRMF, AOBRP, SLIM)
on the Movielens 1M and Movielens 100K datasets.
Bold face means that SMA statistically significantly outperforms the other methods with 95\% confidence level.
}
\centering
\normalsize
\begin{tabular}[t]{|c|c|c|c|c|c|c|c|c|c|}
\hline
\multicolumn{2}{|c|} {Metric} & \multicolumn{4}{c|} {NDCG$@$N} \\
\hline
\multicolumn{2}{|c|}{Data $|$ Method} & N=1 & N=5 & N=10 & N=20\\
\hline
\multirow{6}{*}{\rotatebox{90}{ML-1M}}    & RSVD    & 0.0324 $\pm$ 0.0020 & 0.0700 $\pm$ 0.0006 & 0.0864 $\pm$ 0.0002
                                                    & 0.1006 $\pm$ 0.0001 \\
                                            & BPR   & 0.0538 $\pm$ 0.0006 & 0.1235 $\pm$ 0.0003 & 0.1601 $\pm$ 0.0035
                                                    & 0.2070 $\pm$ 0.0011 \\
                                            & WRMF  & 0.0510 $\pm$ 0.0013 & 0.1202 $\pm$ 0.0002 & 0.1563 $\pm$ 0.0013
                                                    & 0.2012 $\pm$ 0.0010 \\
                                            & AOBPR & 0.0532 $\pm$ 0.0018 & 0.1200 $\pm$ 0.0006 & 0.1567 $\pm$ 0.0009
                                                    & 0.2021 $\pm$ 0.0009 \\
                                            & SLIM & 0.0551 $\pm$ 0.0015 & 0.1201 $\pm$ 0.0023 & 0.1586 $\pm$ 0.0028
                                                    & 0.1948 $\pm$ 0.0043 \\
                                           & \textbf{SMA} & \textbf{0.0729 $\pm$ 0.0007} & \textbf{0.1758 $\pm$ 0.0021}
                          & \textbf{0.2346 $\pm$ 0.0034} & \textbf{0.3002 $\pm$ 0.0048} \\
\hline
\multirow{6}{*}{\rotatebox{90}{ML-100K}}    & RSVD   & 0.0389 $\pm$ 0.0028 & 0.1047 $\pm$ 0.0032 & 0.0996 $\pm$ 0.0059
                                                     & 0.1393 $\pm$ 0.0071 \\
                                            & BPR    & 0.0783 	$\pm$	0.0036
                                                     & 0.1803 	$\pm$	0.0056
                                                     & 0.2351 	$\pm$	0.0056
                                                     & 0.2929 	$\pm$	0.0065 \\
                                            & WRMF   & 0.0913 	$\pm$	0.0034
                                                     & 0.1989 	$\pm$	0.0030
                                                     & 0.2535 	$\pm$	0.0045
                                                     & 0.3131 	$\pm$	0.0043 \\
                                            & AOBPR  & 0.0770 	$\pm$	0.0043
                                                     & 0.1801 	$\pm$	0.0044
                                                     & 0.2343 	$\pm$	0.0051
                                                     & 0.2930 	$\pm$	0.0058 \\
                                            & SLIM   & 0.0912 	$\pm$	0.0021
                                                     & 0.1967 	$\pm$	0.0036
                                                     & 0.2476 	$\pm$	0.0050
                                                     & 0.3017 	$\pm$	0.0091 \\
                                           & \textbf{SMA} & \textbf{0.1013 	$\pm$	0.0041}
                                                     & \textbf{0.2136 	$\pm$	0.0050}
                                                     & \textbf{0.2708 	$\pm$	0.0052}
                                                     & \textbf{0.3321 	$\pm$	0.0056} \\
\hline
\end{tabular}
\label{tab:ndcg}
\end{table*}

Table~\ref{tab:precision} and Table~\ref{tab:ndcg} compare SMA's accuracy in top-N recommendation (Precision@N and NDCG@N)
with one rating-based MA method (RSVD) and four state-of-the-art top-N recommendation methods (BPR, WRMF, AOBRP, SLIM)
on the Movielens 1M and Movielens 100K datasets. Among the compared methods, RSVD is a baseline method,
which aims at minimizing rating-based error. WRMF, BPR and AOBPR are matrix approximation-based
collaborative filtering methods for top-N recommendation. SLIM is not an MA-based method, but we
compare with it due to its superior accuracy in the top-N recommendation task.
Here, we choose exponential loss for SMA with boundary margin $\gamma = 0.3$, rank $r = 200$,
$\lambda_0=\lambda_1=1$, $W_{i,j} = 1$ for positive ratings and $W_{i,j} = 0.03$ for negative ratings.

As shown in the results, SMA consistently and significantly outperforms all the five compared methods
on both datasets in terms of Precision@N and NDCG@N with N varying from 1 to 20. This is because
all the other methods do not consider generalization performance but only optimization performance in model learning,
which cannot ensure optimal test accuracy due to low generalization performance.
Note that WRMF can be regarded as a special case of SMA if we adopt mean square surrogate loss and do not
add subset to improve stability for SMA. Similar to the previous experiments, SMA significantly outperforms
WRMF, which further confirms that introducing properly selected subset can improve the accuracy of
collaborative filtering in the top-N recommendation task.

\section{Related Work}
\label{sec:related}

Algorithmic stability has been analyzed and applied
in several popular problems, such as regression~\cite{Bousquet01},
classification~\cite{Bousquet01}, ranking~\cite{Lan08}, marginal inference~\cite{London13}, etc.
\cite{Bousquet01} first proposed a method of obtaining bounds on
generalization errors of learning algorithms, 
and formally proved that regularization networks posses the uniform stability property. 
Then, \cite{Bousquet02} extends the algorithmic stability concept
from regression to classification.
\cite{Kutin02} generalized the work of \cite{Bousquet01} and proposed
the notion of training stability, which can ensure good generalization error
bounds even when the learner has infinite VC dimension.
\cite{Lan08} proposed query-level stability and gave query-level
generalization bounds to learning to rank algorithms.
\cite{Agarwal09} derived generalization bounds for ranking algorithms that
have good properties of algorithmic stability.
\cite{Shalev-Shwartz10} considered the general learning setting including
most statistical learning problems as special cases, and identified that
stability is the necessary and sufficient condition for learnability.
\cite{London13} proposed the concept of collective stability for structure
prediction, and established generalization bounds for structured prediction.
This work differs from the above works in that  (1) this work introduces the stability concept
to matrix approximation problem, and proves that matrix approximations with
high stability will have high probability to generalize well and
(2) most existing works focus on theoretical analysis, but this work provides
a practical framework for achieving solutions with high stability.

Matrix approximation methods have been extensively studied recently
in the context of collaborative filtering.
\cite{lee2001algorithms} analyzed the optimization problems of Non-negative
Matrix Factorization (NMF). \cite{Srebro04MMMF} proposed Maximum-Margin Matrix
Factorization (MMMF), which can learn low-norm factorizations by solving
a semi-definite program to achieve collaborative prediction.
\cite{mnih2007probabilistic} viewed matrix factorization from a probabilistic perspective
and proposed Probabilistic Matrix Factorization (PMF).
Later, they proposed Bayesian Probabilistic
Matrix Factorization (BPMF)~\cite{salakhutdinov2008bayesian} by giving a fully
Bayesian treatment to PMF. \cite{Lawrence09} also extends PMF and developed
a non-linear PMF using Gaussian process latent variable models.
\cite{paterek2007improving} applied regularized singular value decomposition (RSVD)
in the Netflix Prize contest.
\cite{Koren08} combined matrix factorization and neighborhood model and built
a more accurate combined model named SVD++.
\cite{Hu08} proposed a weighted matrix approximation method for top-N recommendation
on implicit feedback data, which gives higher weights to positive examples and
lower weights for negative examples.
\cite{rendle09} proposed a pair-wise loss function to optimize ranking
measure in top-N recommendation and proposed the BPR method.
Later, they improved the BPR method by proposing a non-uniform item sampler
and oversampling informative pairs to improve convergence speed~\cite{Rendle14}.
Many of the above methods tried to solve overfitting problems in model training,
e.g., regularization in most of the above methods and Bayesian treatment in BPMF.
However, alleviating overfitting cannot decrease the lower bound of generalization errors,
and thus cannot fundamentally solve the low generalization performance problem.
Different from the above works, this work proposes a new optimization
problem with smaller lower bound of generalization error. Minimizing
the new loss function can substantially improve generalization performance
of matrix approximation as demonstrated in the experiments.
\cite{Srebro04} analyzed the generalization error bounds of collaborative
prediction with low-rank matrix approximation for ``0-1'' recommendation.
\cite{Candes10} established error bounds of matrix completion problem with noises.
However, those works did not consider how to achieve matrix approximation
with small generalization error.

Ensemble methods, such as ensemble MMMF~\cite{DeCoste06}, DFC~\cite{mackey2011divide},
LLORMA~\cite{lee2013local}, WEMAREC~\cite{Chen15}, ACCAMS~\cite{Beutel15} etc., have been proposed,
which aimed to provide matrix approximations with high generalization performance
by ensemble learning. However, those ensemble methods need to train a number of
biased weak matrix approximation models, which require much more computations than
SMA. In addition, weak models in those methods are generated by heuristics
which are not directly related to minimizing generalization error. Therefore,
the optimality of generalization performance of those methods cannot be 
proved as in this work.

\section{Conclusion}
\label{sec:conclusion}
Matrix approximation methods are widely adopted in collaborative filtering applications.
However, similar to other machine learning techniques, many
existing matrix approximation methods suffer from the low generalization
performance issue in sparse, incomplete and noisy data, which degrade the stability
of collaborative filtering. This paper introduces the stability notion to the
matrix approximation problem, in which models achieve high
stability will have better generalization performance. Then, SMA, a new 
matrix approximation framework, is proposed to achieve high stability, i.e., high
generalization performance, for collaborative filtering in both rating prediction and
top-N recommendation tasks. Experimental results on real-world datasets
demonstrate that the proposed SMA method can achieve better accuracy than
state-of-the-art matrix approximation methods and ensemble methods in both rating prediction and
top-N recommendation tasks.

\section*{Acknowledgement}
This work was supported in part by the National Natural Science
Foundation of China under Grant No. 61233016, and
the National Science Foundation of USA under Grant Nos.
0954157, 1251257, 1334351, and 1442971.

\bibliographystyle{plain}
\bibliography{ref}

\end{document}